\documentclass{article}
\pdfoutput=1                    % force pdflatex to run
\usepackage[utf8]{inputenc}
\usepackage[sc,osf]{mathpazo}
\usepackage[height=8.6in,width=6.2in,letterpaper]{geometry}
\usepackage[sort&compress,numbers]{natbib}
\usepackage[colorlinks=true,citecolor=blue,breaklinks]{hyperref}
\usepackage[hyphenbreaks]{breakurl} %for arXiv's idiotic hyperref!
\usepackage[tracking]{microtype}

\makeatletter
\newlength\aftertitskip     \newlength\beforetitskip
\newlength\interauthorskip  \newlength\aftermaketitskip

\def\maketitle{\par
 \begingroup
   \def\thefootnote{\fnsymbol{footnote}}
   \def\@makefnmark{\hbox to 4pt{$^{\@thefnmark}$\hss}}
   \@maketitle \@thanks
 \endgroup
\setcounter{footnote}{0}
 \let\maketitle\relax \let\@maketitle\relax
 \gdef\@thanks{}\gdef\@author{}\gdef\@title{}\let\thanks\relax}

\def\@startauthor{\noindent \normalsize\bf}
\def\@endauthor{}
\def\@starteditor{\noindent \small {\bf Editor:~}}
\def\@endeditor{\normalsize}
\def\@maketitle{\vbox{\hsize\textwidth
 \linewidth\hsize \vskip \beforetitskip
 {\begin{center} \LARGE\@title \par \end{center}} \vskip \aftertitskip
 {\def\and{\unskip\enspace{\rm and}\enspace}%
  \def\addr{\small\it}%
  \def\email{\hfill\small\tt}%
  \def\name{\normalsize\bf}%
  \def\AND{\@endauthor\rm\hss \vskip \interauthorskip \@startauthor}
  \@startauthor \@author \@endauthor}
}}

\makeatother

\usepackage[colorlinks=true,citecolor=blue]{hyperref}
\usepackage{amsmath,amssymb,amsthm}
\usepackage{xcolor}
\usepackage[sort&compress,numbers]{natbib}

\usepackage{graphicx}
\usepackage{caption}
\usepackage{subcaption}
\usepackage{appendix}
\usepackage{enumerate}
\usepackage{wrapfig}

\newcommand{\btau}{\bar{\tau}}

              \newcommand{\Oc}{\mathcal{O}}         \newcommand{\Xc}{\mathcal{X}}   

\newcommand{\xavg}{\bar{x}}

\newcommand{\norm}[1]{\|{#1}\|}

\newcommand{\nlsum}{\sum\nolimits}

\newcommand{\ip}[2]{\langle {#1},\, {#2} \rangle}

\newcommand{\set}[1]{\{ #1\}}

\newcommand{\reals}{\mathbb{R}}

\newcommand{\E}{\mathbb{E}}

\newcommand{\pp}{\mathbb{P}}

\newcommand{\half}{\tfrac{1}{2}}

\newcommand{\xqedhere}[2]{%
  \rlap{\hbox to#1{\hfil\llap{\ensuremath{#2}}}}}

\newcommand{\Perp}{\mathrel{\text{\scalebox{0.85}{$\perp\mkern-10mu\perp$}}}}

\DeclareMathOperator*{\argmin}{argmin}

\newtheorem{theorem}{Theorem}
\newtheorem{lemma}[theorem]{Lemma}

\newtheorem{corr}[theorem]{Corollary}
\theoremstyle{definition}
\newtheorem{defn}[theorem]{Definition}

\newtheorem{assum}[theorem]{Assumption}

\numberwithin{equation}{section}
\numberwithin{theorem}{section}

\title{AdaDelay: Delay Adaptive Distributed\\ Stochastic Convex Optimization}
\author{\name Suvrit Sra \email{suvrit@mit.edu}\\
  \addr{Massachusetts Institute of Technology}
  \AND
  \name Adams Wei Yu  \email{weiyu@cs.cmu.edu}\\
  \addr{Carnegie Mellon University}
  \AND 
  \name Mu Li \email{muli@cs.cmu.edu}\\
  \addr{Carnegie Mellon University}
  \AND
  \name Alexander J. Smola \email{alex@smola.org}\\
  \addr{Carnegie Mellon University}
}

\begin{document}
\maketitle

\begin{abstract}
  We study distributed stochastic convex optimization under the delayed gradient model where the server nodes perform parameter  updates, while the worker nodes compute stochastic gradients. We  discuss, analyze, and experiment with a setup motivated by the  behavior of real-world distributed computation networks, where the  machines are differently slow at different time. Therefore, we  allow the parameter updates to be sensitive to the actual delays  experienced, rather than to worst-case bounds on the maximum  delay. This sensitivity leads to larger stepsizes, that can help   gain rapid initial convergence without having to wait too long for  slower machines, while maintaining the same asymptotic  complexity. We obtain encouraging improvements to overall  convergence for distributed experiments on real datasets with up to  billions of examples and features.
\end{abstract}

\section{Introduction}
\label{sec:intro}
We study the stochastic convex optimization problem
\begin{equation}
  \label{eq:1}
  \min_{x \in \Xc}\quad f(x) := \E[F(x; \xi)],
\end{equation}
where the constraint set $\Xc \subset \reals^d$ is a compact convex set, and $F(\cdot,\xi)$ is a convex loss for each $\xi \sim \pp$, where $\pp$ is an unknown probability distribution from which we can draw i.i.d.\ samples. Problem~\eqref{eq:1} is broadly important in both optimization and machine learning~\citep{sreTew10,shapiro,shamir2014,nemirov09,ghadimi2012}. It should be distinguished from (and is harder than) the finite-sum optimization problems~\citep{Schmidt13,bertsekas2011}, for which sharper results on the empirical loss are possible but not on the generalization error. 

A classic approach to solve~\eqref{eq:1} is via stochastic gradient descent (SGD)~\citep{RobMon51} (also called stochastic approximation~\citep{nemirov09}). At each iteration SGD performs the update
$x(t+1) \gets \Pi_{\Xc}(x(t) - \alpha_t g(x(t)))$, where $\Pi_{\Xc}$ denotes orthogonal projection onto $\Xc$, scalar $\alpha_t \ge 0$ is a suitable stepsize, and $g(x(t))$ is an unbiased stochastic gradient such that $\E[g(x(t))] \in \partial f(x(t))$.

Although much more scalable than gradient descent, SGD is still a sequential method that cannot be immediately used for truly large-scale problems requiring distributed optimization. Indeed, distributed optimization~\citep{BerTsi89} is a central focus of real-world machine learning, and has attracted significant recent research interest, a large part of which is dedicated to scaling up SGD~\citep{agDuc11,ram2010,duchi2012dual,li2014b,LanSmoZin09}.

\paragraph{Motivation.} Our work is motivated by the need to more precisely model and exploit the delay properties of real-world cloud providers; indeed, the behavior of machines and delays in such settings are typically quite different from what one may observe on small clusters owned by individuals or small groups. In particular, cloud resources are shared by many users who run variegated tasks on them. Consequently, such an environment will invariably be more diverse in terms of availability of key resources such as CPU, disk, or network bandwidth, as compared to an environment where resources are shared by a small number of individuals. Thus, being able to accommodate for variable delays is of great value to both providers and users of large-scale cloud services.

In light of this background, we investigate delay sensitive asynchronous SGD, especially, for being able to adapt to the actual delays experienced rather than using global upper-case `bounded delay' arguments that can be too pessimistic. A potential practical approach is as follows: in the beginning the server updates parameters whenever its receives a gradient from any machine, with a weight inversely proportional to the actual delay observed. Towards the end, the server may take larger update steps whenever it gets a gradient from a machine that sends parameters infrequently, and small ones if it get parameters from a machine that updates very frequently, to reduce the bias caused by the initial aggressive steps.

\paragraph{Contributions.} The key contributions of this paper are underscored by our practical motivation. In particular, we design, analyze and investigate \textit{AdaDelay} (\textbf{Ada}ptive \textbf{Delay}), an asynchronous SGD algorithm, that more closely follows the actual delays experienced during computation. Therefore, instead of penalizing parameter updates by using worst-case bounds on delays, AdaDelay uses step sizes that depend on the actual delays observed. While this allows the use of larger stepsizes, it requires a slightly more intricate analysis because (i) step sizes and are no longer guaranteed to be monotonically decreasing; and (ii) residuals that measure progress are not independent across time as they are coupled by the delay random variable. 

We validate our theoretical framework by experimenting with large-scale machine learning datasets containing over a billion points and features. The experiments reveal that our assumptions of network delay are a reasonable approximation to the actual observed delays, and that in the regime of large delays (e.g., when there are stragglers), using delay sensitive steps is very helpful toward obtaining models that more quickly converge on \emph{test accuracy}; this is revealed by experiments where using AdaDelay leads to significant improvements on the test error (AUC). %, and we hope that this encourages a wider investigation of delay adaptive models (see also~\citep{mcmahan2014}).

\paragraph{Related Work.} 
An useful summary on aspects of stochastic optimization in machine learning is~\citep{sreTew10}; more broadly, \citep{nemirov09,shapiro} are excellent references. Our focus is on distributed stochastic optimization in the asynchronous setting. The classic work~\citep{BerTsi89} is an important reference; more recent works closest to ours are~\citep{agDuc11,shamir2014,LanSmoZin09,nedic2001distributed}. Of particular relevance to our paper is the recent work on delay adaptive gradient scaling in an AdaGrad like framework~\citep{mcmahan2014}. The work~\citep{mcmahan2014} claims substantial improvements under specialized settings over~\citep{duchi2013}, a work that exploits data sparsity in a distributed asynchronous setting. Our experiments confirm \citep{mcmahan2014}'s claims that their best learning rate is insensitive to maximum delays. However, in our experience the method of~\citep{mcmahan2014} overly smooths the optimization path, which can have adverse effects on real-world data (see Section~\ref{sec:exp}). 

To our knowledge, all previous works on asynchronous SGD (and its AdaGrad variants) assume monotonically diminishing step-sizes. Our analysis, although simple, shows that rather than using worst case delay bounds, using exact delays to control step sizes can be remarkably beneficial in realistic settings: for instance, when there are stragglers that can slow down progress for all the machines in a worst-case delay model.

Algorithmically, the work~\citep{agDuc11} is the one most related to ours; the authors of~\citep{agDuc11} consider using delay information to adjust the step size. However, the most important difference is that they only use the worst possible delays which might be too conservative, while AdaDelay leverages the actual delays experienced. \citep{LanSmoZin09} investigates two variants of update schemes, both of which occur with delay. But they do not exploit the actual delays either.
There are some other interesting works studying specific scenarios, for example, \cite{duchi2013}, which focuses on the sparse data. However, our framework is more general and thus capable of covering more applications.

\section{Problem Setup and Algorithm}\label{sec:setup}
We build on the groundwork laid by~\citep{agDuc11,nedic2001distributed}; like them, we also consider optimizing~\eqref{eq:1} under a delayed gradient model. The computational framework that we use is the parameter-server~\citep{li2014b}, so that a central server\footnote{This server is virtual; its physical realization may involve several machines, e.g.,~\citep{li2014}.} maintains the global parameter, and the worker nodes compute stochastic gradients using their share of the data. The workers communicate their gradients back to the central server (in practice using various crucial communication saving techniques implemented in the work of~\citep{li2014}), which updates the shared parameter and communicates it back. 

To highlight our key ideas and avoid getting bogged down in unilluminating details, we consider only smooth stochastic optimization, i.e., $f \in C_L^1$ in this paper. Straightforward, albeit laborious extensions are possible to nonsmooth problems, strongly convex costs, mirror descent versions, proximal splitting versions. Such details are relegated to a longer version of this paper.

Specifically, following~\citep{agDuc11,nemirov09,ghadimi2012} we also make the following standard assumptions:\footnote{These are easily satisfied for logistic-regression, least-squares, if the training data are bounded.}
\begin{assum}[Lipschitz gradients]
  \label{ass:lip}
  The function $f$ has a locally $L$-Lipschitz  gradients. That is,
  \begin{equation*}
    \norm{\nabla f(x)-\nabla f(y)} \le L\norm{x-y},\qquad \forall\ \ x, y \in \Xc.
  \end{equation*}
\end{assum}
\begin{assum}[Bounded variance]
  \label{ass:var}
  There exists a constant
  $\sigma < \infty$ such that 
  \begin{equation*}
    \E_\xi[\norm{\nabla f(x) - \nabla F(x; \xi)}^2] \le \sigma^2,\qquad \forall\ x \in \Xc.
  \end{equation*}
\end{assum}
\begin{assum}[Compact domain]
  \label{ass:dom}
  Let $x^* \in \argmin_{x \in \Xc}  f(x)$. Then, 
  \begin{equation*}
    \max_{x \in \Xc} \norm{x-x^*} \le R.
  \end{equation*}
\end{assum}
Finally, an additional assumption, also made in~\citep{agDuc11} is that of bounded gradients.
\begin{assum}[Bounded Gradient]
  \label{ass:bound_grad}
  Let $\forall\ \ x \in \Xc$. Then,
  \begin{equation*}
    \|\nabla f(x)\| \le G.
  \end{equation*}
\end{assum}
These assumptions are typically reasonable for machine learning problems, for instance, logistic-regression losses and least-squares costs, as long as the data samples $\xi$ remain bounded, which is typically easy to satisfy. Exploring relaxed versions of these assumptions would also be interesting.

\paragraph{Notation:} We denote a random delay at time-point $t$ by $\tau_t$; step sizes are denoted $\alpha(t,\tau_t)$, and delayed gradients as $g(t-\tau_t)$. For a differentiable convex function $h$, the corresponding \emph{Bregman divergence} is $D_h(x,y) := h(x)-h(y)-\ip{\nabla h(y)}{x-y}$. For simplicity, all norms are assumed to be Euclidean. We also interchangeably use $x_t$ and $x(t)$ to refer to the same quantity.

\subsection{Delay model}
\begin{assum}[Delay]
  \label{ass:delay}
  We consider the following two practical delay models:
\begin{enumerate}[(A)]
\item \textbf{Uniform:} Here $\tau_t \sim U(\set{0,2\btau})$. This model is a reasonable approximation to observed delays after an initial startup time of the network. We could make a more refined assumption that for iterations $1\le t \le T_1$, the delays are uniform on $\set{0,\ldots,T_1-1}$, and the analysis easily extends to handle this case; we omit it for ease of presentation. Furthermore, the analysis also extends to delays having distributions with bounded support. Therefore, it indeed captures a wide spectrum of practical models.
\item \textbf{Scaled:}  For each $t$, there is a $\theta_t \in (0,1)$ such that $\tau_t < \theta_t t$. Moreover, assume that
  \begin{equation*}
    \E[\tau_t] = \btau_t,\qquad \E[\tau_t^2] = B_t^2,
  \end{equation*}
  are constants that do not grow with $t$ (the subscript only indicates that each $\tau_t$ is a random variable that may have a different distribution). This model allows delay processes that are richer than uniform as long as they have bounded first and second moments.
\end{enumerate}
\end{assum}
\textbf{Remark:} 
Our analysis seems general enough to cover many other delay distributions by combining our two delay models. For example,
the Gaussian model (where $\tau_t$ obeys a Gaussian distribution but its support must be truncated as $t>0$) may be seen as a combination of the following:
1) When $t \ge C$ (a suitable constant), the Gaussian assumption indicates $\tau_t < \theta t$, which falls under our second delay model;
2) When $0\le t\le C$, our proof technique with bounded support (same as uniform model) applies.
%%In fact, Fig \ref{fig:auc2} indeed shows that our algorithm also handles delays with stragglers.
Of course, we believe a more refined analysis for specific delays may help tighten constants.

\subsection{Algorithm}
Under the above delay model, we consider the following projected stochastic gradient iteration:
\begin{equation}
  \label{eq:iter1}
  x({t+1}) \gets \argmin_{x \in \Xc}\Bigl[\ip{g(t-\tau_t)}{x}
    +           \frac{1}{2\alpha(t,\tau_t)}\norm{x-x(t)}^2\Bigr],\qquad t=1,2,\ldots,
\end{equation}
where the stepsize $\alpha(t,\tau_t)$ is sensitive to the actual delay observed.
Iteration~\eqref{eq:iter1} generates a sequence $\set{x(t)}_{t\ge 1}$; the server also maintains the averaged iterate
\begin{equation}
  \label{eq:avg}
  \xavg(T) := \frac{1}{T}\sum_{t=1}^Tx(t+1). %,\qquad \Xavg(T) = \frac{2}{(T+1)(T+2)}\sum_{t=1}^T(t+1)x(t+1).
\end{equation}

\section{Analysis}
\label{sec:analysis}
We use stepsizes of the form $\alpha(t,\tau_t) = (L + \eta(t,\tau_t))^{-1}$, where the step offsets $\eta(t,\tau_t)$ are chosen to be sensitive to the actual delay of the incoming gradients. We typically use
\begin{equation}
  \label{eq:eta}
  \eta(t,\tau_t) = c\sqrt{t + \tau_t},
\end{equation}
for some constant $c$ (to be chosen later). Actually, we can also consider time-varying $c_t$ multipliers in~\eqref{eq:eta} (see Corollary~\ref{corr:bound_ct}), but initially for clarity of presentation we let $c$ be independent of $t$. Thus, if there are no delays, then $\tau_t = 0$ and iteration~\eqref{eq:iter1} reduces to the usual synchronous SGD. The constant $c$ is used to tradeoff contributions in the error bound from the noise variance $\sigma$, the radius bound $R$, and potentially bounds on gradient norms.

Our convergence analysis builds heavily on~\citep{agDuc11}. But the key difference is that our step sizes $\alpha(t,\tau_t)$ depend on the actual delay $\tau_t$ experienced, rather than on a fixed worst-case bounds on the maximum possible network delay. These delay dependent step sizes necessitate a slightly more intricate analysis. The primary complexity arises from $\alpha(t,\tau_t)$ being no longer independent of the actual delay $\tau_t$. This in turn raises another difficulty, namely that $\alpha(t,\tau_t)$ are \emph{no longer} monotonically decreasing, as is typically assumed in most convergence analyses of SGD.
We highlight our theoretical result below, and due to space limitation, all the auxiliary lemmas are moved to appendix.

\begin{theorem}
  \label{thm:convg1}
  Let $x(t)$ be generated according to~\eqref{eq:iter1}. Under Assumption \ref{ass:delay} (A) (uniform delay) we have
  \begin{align*}
    \E\left[\nlsum_{t=1}^T \bigl(f(x(t+1)) -f(x^*)\bigr) \right] 
    \le & \left(\sqrt{2}cR^2\btau + \frac{\sigma^2}{c} \right)\sqrt{T} +{LG^2(4\btau+3)(\btau+1)\over 6c^2}\log T  \\
    & + \half (L+c)R^2 +
    \btau GR + {LG^2 \btau(\btau+1)(2\btau+1)^2\over 6(L^2+c^2)} ,
  \end{align*}
  while under  Assumption~\ref{ass:delay}~(B)  (scaled delay) we have
  \begin{align*}
    \E\left[\nlsum_{t=1}^T \bigl(f(x(t+1)) -f(x^*)\bigr) \right] 
    \le & \frac{\sigma^2}{c}\sqrt{T} + {1\over 2} cR^2 \sum_{t=2}^{T} \frac{\btau_t+1}{\sqrt{2t-1}} 
    + GR\left[1+ \sum_{t=1}^{T-1}\frac{B_t^2}{(T-t)^2} \right]\\
    & + G^2\sum_{t=1}^T\frac{B_t^2+1+\btau_t}{L^2+c^2(1-\theta_t)t} + {1\over 2} R^2(L+c).
  \end{align*}
\end{theorem}
\begin{proof}[Proof Sketch.]
  The proof begins by analyzing the difference $f(x(t+1))-f(x^*)$; Lemma~\ref{lem:e} (provided in the supplement) bounds this difference, ultimately leading to an inequality of the form:
  \begin{equation*}
    \E\left[\nlsum_{t=1}^T \bigl(f(x(t+1)) -f(x^*)\bigr) \right] 
    \le
    \E\left[\nlsum_{t=1}^T\Delta(t) + \Gamma(t) + \Sigma(t)\right].
  \end{equation*}
  The random variables $\Delta(t)$, $\Gamma(t)$, and $\Sigma(t)$ are in turned given by
  \begin{align}
    \label{eq:28}
    \Delta(t)  &:= \frac{1}{2\alpha(t,\tau_t)}\left[\norm{x^*-x(t)}^2-\norm{x^*-x(t+1) }^2\right];\\
    \label{eq:29}
    \Gamma(t)  &:= \ip{\nabla f(x(t))-\nabla f(x(t-\tau_t))}{x(t+1) -x^*};\\
    \label{eq:30}
    \Sigma(t)  &:= \tfrac{1}{2\eta(t,\tau_t)}\norm{\nabla f(x(t-\tau_t))-g(t-\tau_t)}^2.
  \end{align}
  Once we bound these in expectation, we obtain the result claimed in the theorem. In particular, 
  Lemma~\ref{lem:delta1} bounds \eqref{eq:28} under Assumption~\ref{ass:delay}(A), while Lemma~\ref{lem:delta2} provides a bound under the Assumption~\ref{ass:delay}(B). Similarly, Lemmas~\ref{lem:gamma1} and Lemma~\ref{lem:gamma2} bounds~\eqref{eq:29}, while Lemma~\ref{lem:sigma}  bounds~\eqref{eq:30}. Combining these bounds we obtain the theorem.
\end{proof}

Theorem~\ref{thm:convg1} has several implications. Corollaries~\ref{corr:uniform} and~\ref{corr:linear} indicate both our delay models share a similar convergence rate, while Corollary~\ref{corr:bound_ct} shows such results still hold even we replace the constant $c$ with a  bounded (away from zero, and from above) sequence $\{c_t\}$ (a setting of great practical importance). Finally, Corollary~\ref{corr:beta} mentions in passing a simple variant that considers $\eta_t = c_t (t+\tau_t)^{\beta}$ for $0<\beta<1$.
It also highlights the known fact that for $\beta=0.5$, the algorithm achieves the best theoretical convergence. 

\begin{corr}
\label{corr:uniform}
%  Let $\xavg$ be the averaged iterate, and 
  Let $\tau_t$ satisfy Assumption~\ref{ass:delay} (A). 
%  With the choice $\eta(t,\tau_t) = \sigma\sqrt{t+\tau_t}/R$ 
  Then we have
  \begin{align*}
    \E[ f(\xavg_T)-f^*] = \Oc\left(D_1\frac{\sqrt{T}}{T}+D_2\frac{\log T}{T}+ D_3\frac{1}{T}\right).
  \end{align*}
where
   $$D_1 =  \sqrt{2}cR^2\btau + \frac{\sigma^2}{c} ,
   D_2 = {LG^2(4\btau+3)(\btau+1)\over 6c^2}, 
   D_3 = \half (L+c)R^2 +
         \btau GR + {LG^2 \btau(\btau+1)(2\btau+1)^2\over 6(L^2+c^2)}.$$
\end{corr}

The following corollary follows easily from combining Theorem~\ref{thm:convg1} with Lemma~\ref{lem:gamma2}.
\begin{corr}
\label{corr:linear}
  Let $\tau_t$ satisfy Assumption~\ref{ass:delay}~(B); let $\btau_t=\tau$, $\theta_t = \theta$, and $B_t=B$ for all $t$. Then,
   \begin{align*}
     \E[ f(\xavg_T)-f^*] = \Oc\left(D_4\frac{\sqrt{T}}{T}+D_5\frac{\log \bigl(1+\frac{c^2(1-\theta)T}{L^2}\bigr)}{T}+ D_6\frac{1}{T}\right).
   \end{align*}
where 
$$D_4 = \left[{1\over \sqrt{2}} cR^2(\btau+1)+{\sigma^2\over c}\right], 
D_5 = \frac{G^2(B^2+\tau+1)}{c^2(1-\theta)},
D_6 = \half (L+c)R^2 + GR\left(1+{\pi^2B^2\over 6}\right) .$$
\end{corr}

\begin{corr}
\label{corr:bound_ct}
If  $\eta_t = c_t \sqrt{t+\tau_t}$ with $0<M_1\le c_t\le M_2$, then the conclusion of Theorem~\ref{thm:convg1}, Corollary~\ref{corr:uniform} and~\ref{corr:linear} still hold, except that the term $c$ is replaced by $M_2$ and $1\over c$ by $1\over M_1$.   
\end{corr}

If we wish to use step size offsets $\eta_t=c_t(t+\tau_t)^\beta$ where $0 <\beta < 1$, we get a result of the form (we report only the asymptotically worse term, as this result is of limited importance).
\begin{corr}
\label{corr:beta}
Let  $\eta_t = c_t (t+\tau_t)^{\beta}$ with $0<M_1\le c_t\le M_2$ and $0<\beta<1$. Then, there exists a constant $D_7$ such that
$$ \E[ f(\xavg_T)-f^*] = \Oc\left(\frac{D_7}{T^{\min(\beta, 1-\beta)}} \right).$$ 
\end{corr}

\section{Experiments}
\label{sec:exp}
We now evaluate the efficiency of AdaDelay in a distributed environment using real datasets.

\paragraph{Setup.} We collected two click-through rate datasets for
evaluation, which are shown in Table~\ref{tab:dataset}. One is the Criteo
dataset\footnote{\url{http://labs.criteo.com/downloads/download-terabyte-click-logs/}},
where the first 8 days are used for training while the following 2 days are used for
validation. We applied one-hot encoding for category and string features. The
other dataset, named CTR2, is collected from a large Internet company. We
sampled 100 million examples from three weeks for training, and 20 millions examples from the
next week for validation. We extracted 2 billion unique features using the
on-production feature extraction module. These two datasets have comparable
size, but different example-feature ratios.
We adopt Logistic Regression as our classification model.

\begin{table}[htbp]
  \centering
  \begin{tabular}{|r|cccc|}
    \hline
   & training example & test example & unique feature & non-zero entry \\
    \hline
    Criteo & 1.5 billion & 400 million & 360 million & 58 billion \\
    CTR2 & 110 million & 20 million & 1.9 billion & 13 billion  \\
    \hline
  \end{tabular}
  \caption{Click-through rate datasets.}
  \label{tab:dataset}
\end{table}

All experiments were carried on a cluster with 20 machines. Most machines are
equipped with dual Intel Xeon 2.40GHz CPUs, 32 GB memory and 1 Gbit/s Ethernet.

\paragraph{Algorithm.}
We compare AdaDelay with two related methods AsyncAdaGrad \cite{agDuc11} and
AdaptiveRevision \cite{mcmahan2014}. Their main difference lies in the choice of
the learning rate at time $t$: $\alpha(t,\tau_t) =
(L+\eta(t,\tau_t))^{-1}$. Denote by $\eta_j(t,\tau_t)$ the $j$-th element of
$\eta(t,\tau_t)$, and similarly $g_j(t - \tau_t)$ the delayed gradient on
feature $j$.  AsyncAdaGrad adopts a scaled learning rate $\eta_j(t, \tau_t) =
\sqrt{\sum_{i=1}^{t} g_j^2(i, \tau_i) }$. AdaptiveRevision takes into account
actual delays by considering $g^{\textrm{bak}}_j(t, \tau_t) =
\sum_{i=t-1}^{t-\tau} g_j(i, \tau_i)$. It uses a non-decreasing learning rate
based on $\sqrt{\sum_{i=1}^{t} g_j^2(i, \tau_i) + 2 g_j(t, \tau_t) g^{\textrm{bak}}_j(t, \tau_t)}$.
Similar to AsyncAdaGrad and AdaptiveRevision, we use a scaled learning
rate in AdaDelay to better model the nonuniform sparsity of the
dataset (this step size choice falls within the purview of
Corollary~\ref{corr:bound_ct}). In other words, we set
$\eta_j(t,\tau_t) = c_j \sqrt{t+\tau_t}$, where
$c_j = \sqrt{\frac1t\sum_{i=1}^{t} \frac{i}{i+\tau_i}g_j^2(i-\tau_i)}$
averages the weighted delayed gradients on feature $j$. We follow the
common practice of fixing $L$ to 1 while choosing the best
$\alpha(t,\tau_t) = \alpha_0 (L+\eta(t,\tau_t))^{-1}$ by a grid
search over $\alpha_0$.

\paragraph{Implementation.}
We implemented these three methods in the parameter server framework
\cite{li2014b}, which is a high-performance asynchronous communication library
supporting various data consistency models. There are two groups of nodes in
this framework: workers and servers.  Worker nodes run independently from each
other. At each time, a worker first reads a minibatch of data from a distributed
filesystem, and then pulls the relevant recent working set of parameters, namely the
weights of the features that appear in this minibatch, from the server
nodes. It next computes the gradients and then pushes these gradients to the
server nodes.

The server nodes maintain the weights. For each feature, both AsyncAdaGrad and
AdaDelay store the weight and the accumulated gradient which is used to compute the
scaled learning rate. While AdaptiveRevision needs two more entries for
each feature.

To compute the actual delay $\tau$ for AdaDelay, we let the server nodes record
the time $t(w, i)$ when worker $w$ is pulling the weight for minibatch
$i$. Denote by $t'(w,i)$ the time when the server nodes are updating the weight
by using the gradients of this minibatch. Then the actual delay of this minibatch
can be obtained by $t'(w,i) - t(w,i)$.

AdaptiveRevision needs gradient components $g^{\textrm{bck}}_j$ for each feature $j$ to calculate
its learning rate. If we send $g^{\textrm{bck}}_j$ over the network by following
\cite{mcmahan2014}, we increase the total network communication by $50\%$, which
harms the system performance due to the limited network bandwidth. Instead, we
store $g^{\textrm{bck}}_j$ at the server node during while processing this
minibatch. This incurs no extra network overhead, however, it increases the memory
consumption of the server nodes.

The parameter server implements a node using an operation system process, which
has its own communication and computation threads. In order to run thousands of
workers on our limited hardware, we may combine server workers into a single
process to reduce the system overhead.

\paragraph{Results.}

We first visualize the actual delays observed at server nodes. As noted from
Figure~\ref{fig:delay}, delay $\tau_t$ is around $\theta t$ at the early stage
of the training, while the constant $\theta$ varies for different tasks. For example, it is
close to $0.2$ when training the Criteo dataset with 1,600 workers, while it
increases to $1$ for the CTR2 dataset with 400 workers. After the delay hitting the
value $u$, which is often half of the number of workers, it
behaves as a Gaussian distribution with mean $u$, which are shown in the bottom
of Figure~\ref{fig:delay}.

\begin{figure}[th!]
  \centering
 \begin{subfigure}[b]{\textwidth}
   \centering
   \includegraphics[width=.4\textwidth]{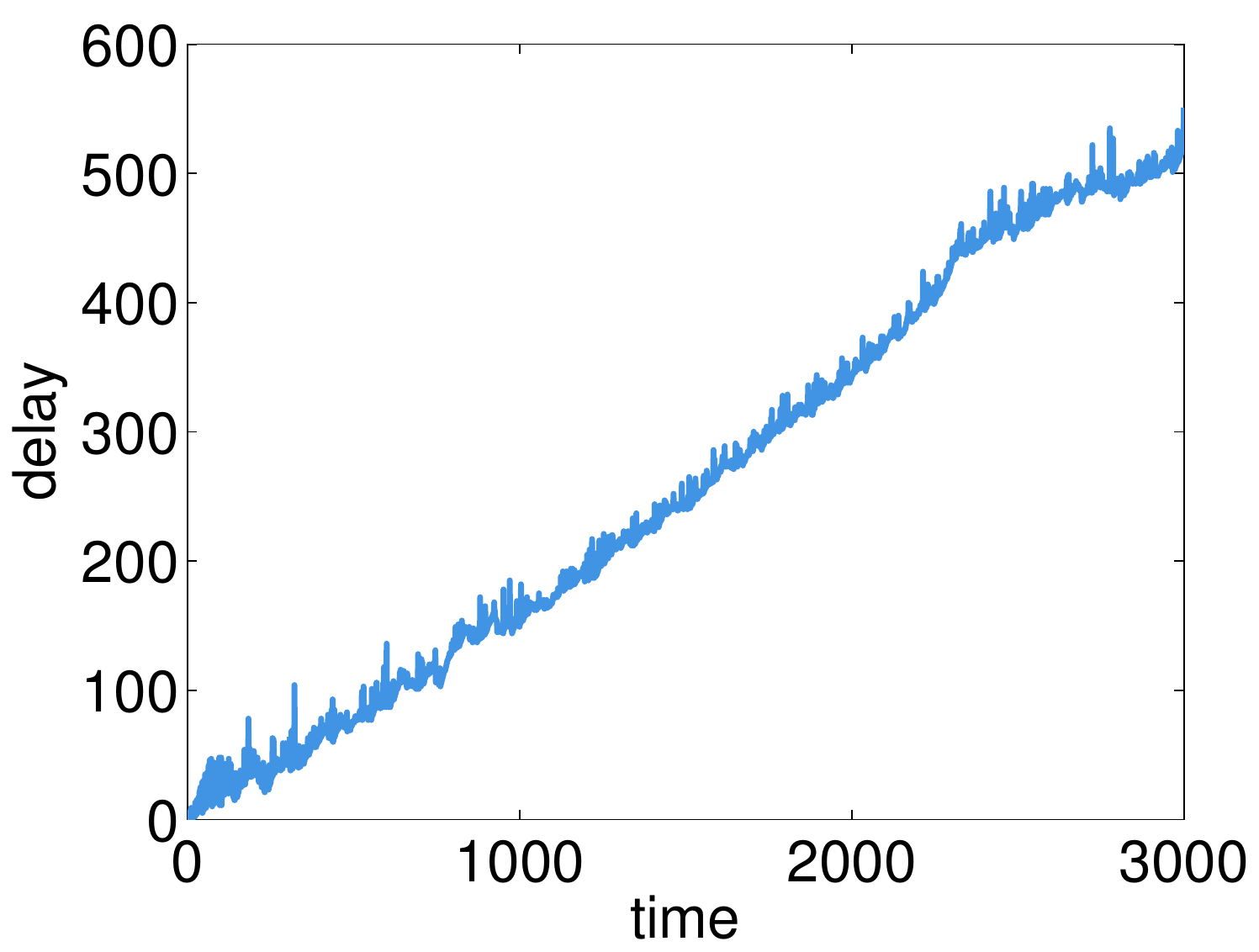}\hspace{4ex}%
  \includegraphics[width=.4\textwidth]{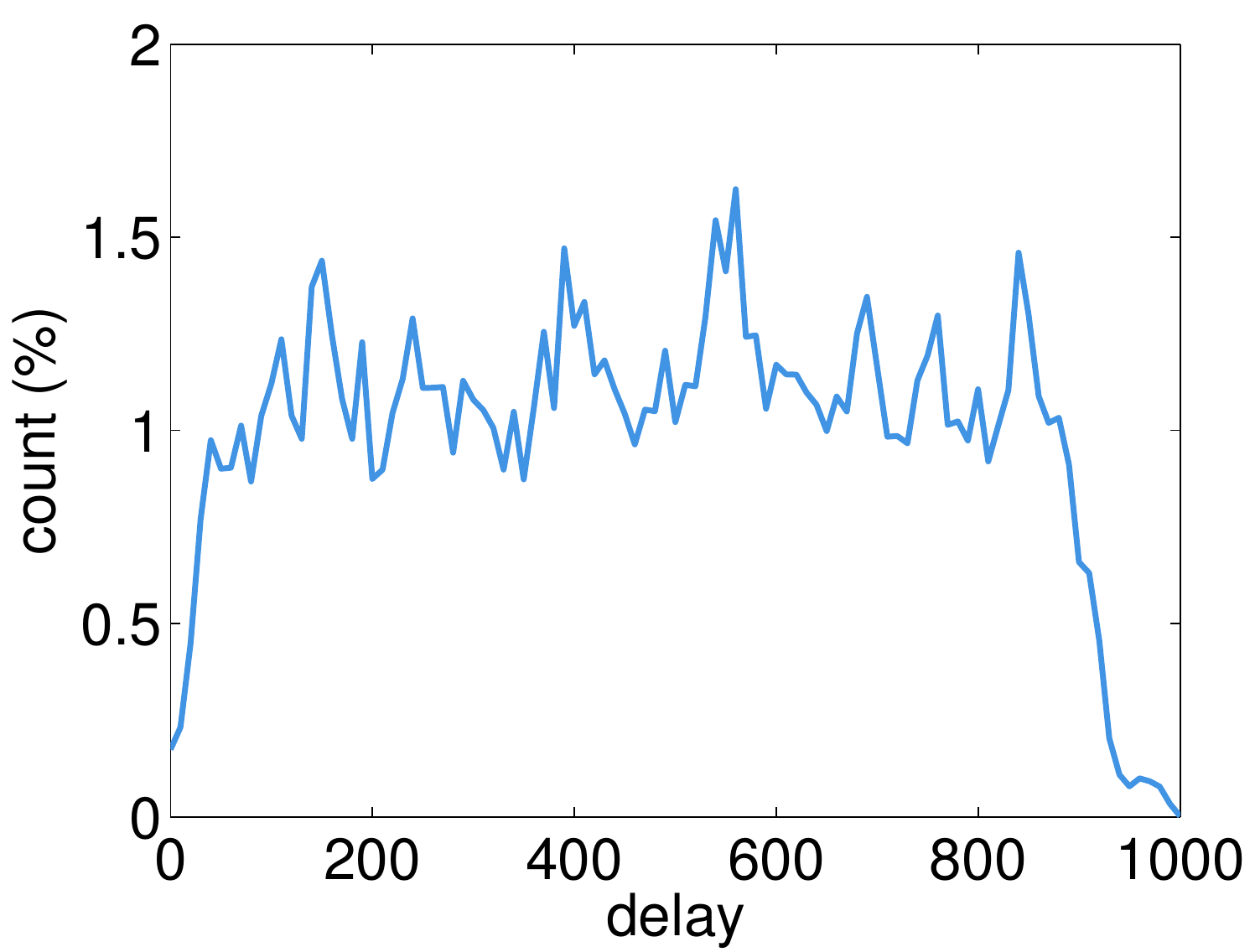}\\%
\caption{Criteo dataset with 1,600 workers}
   \end{subfigure}
 \begin{subfigure}[b]{\textwidth}
   \centering
  \includegraphics[width=.4\textwidth]{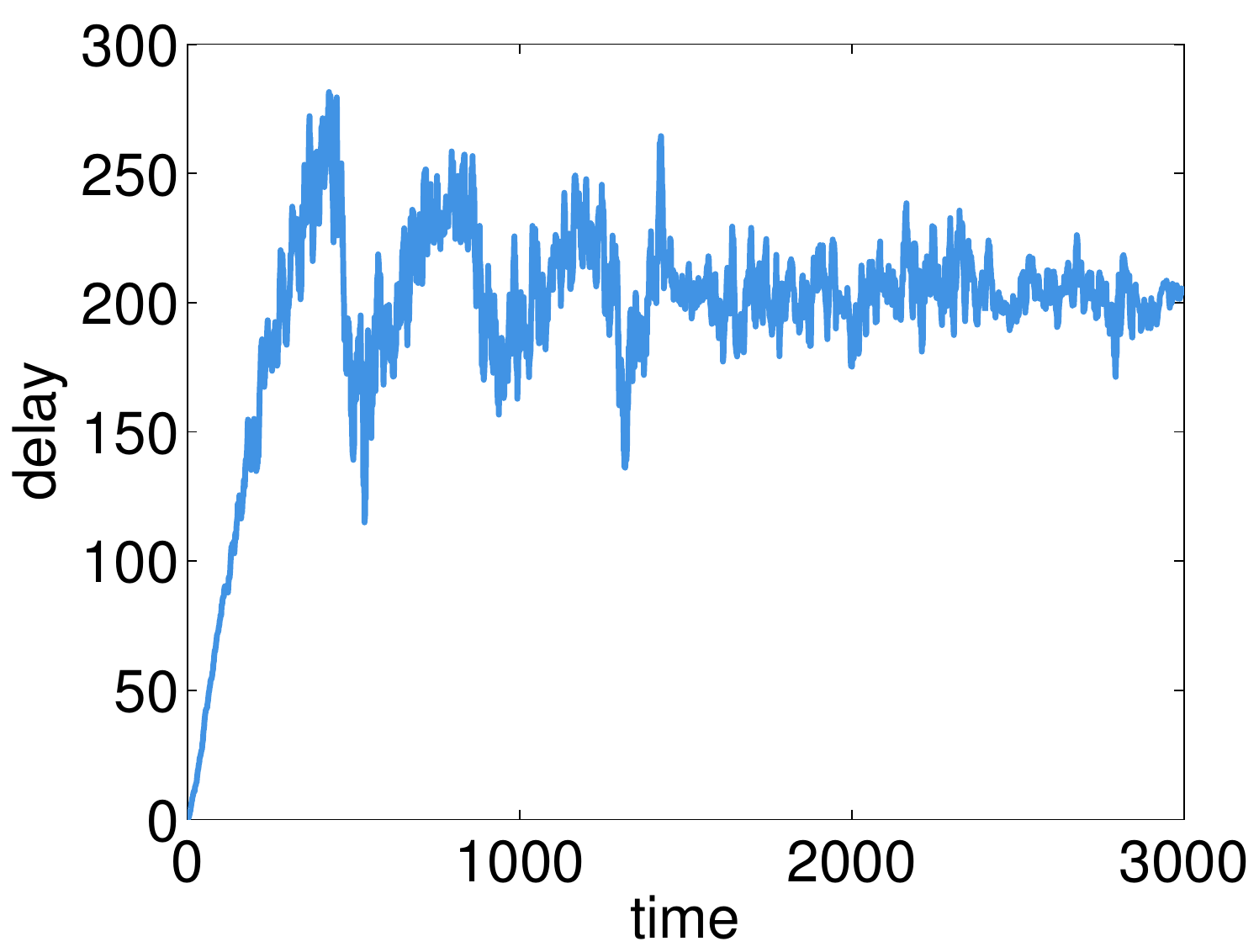}\hspace{4ex}%
  \includegraphics[width=.4\textwidth]{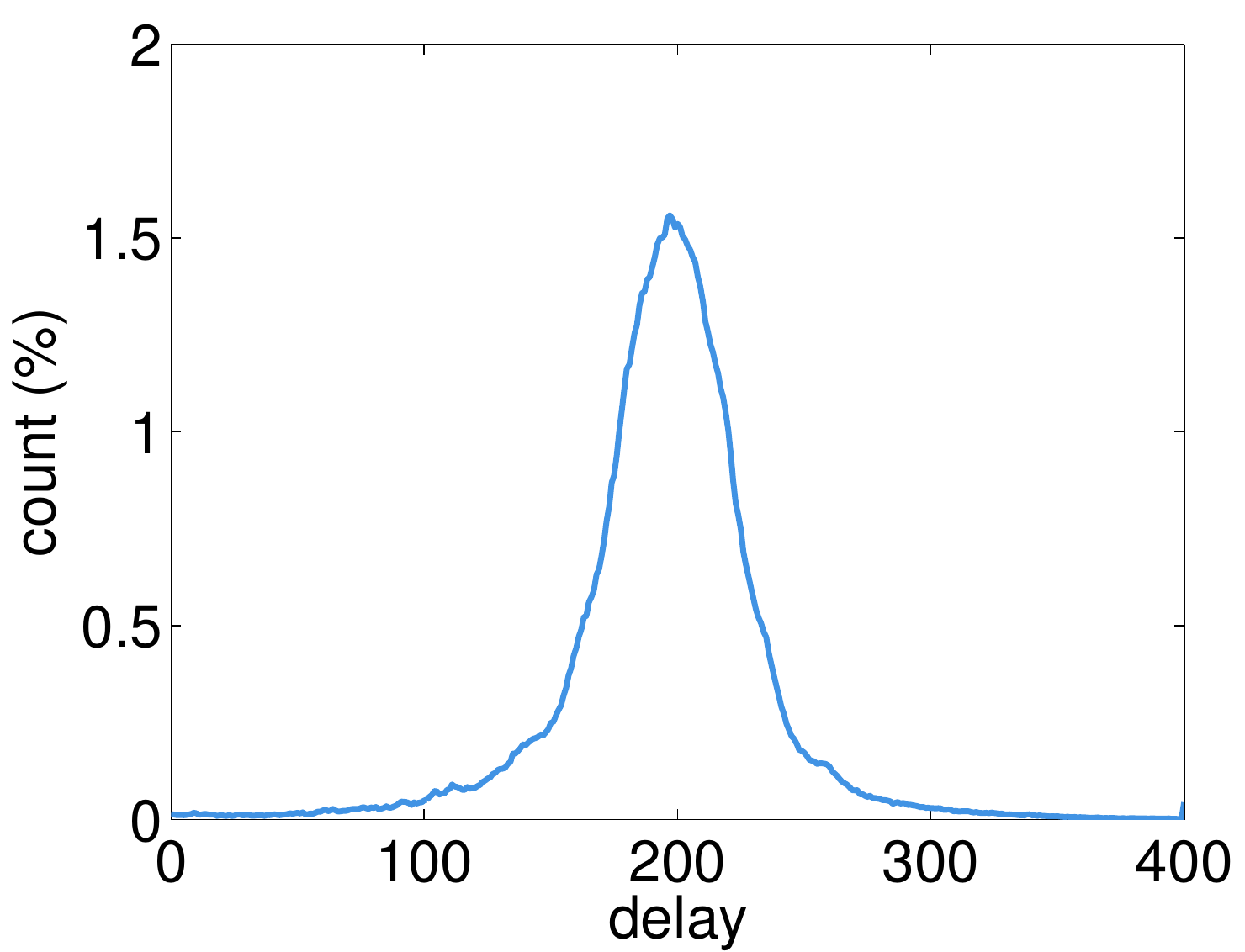}
\caption{CTR2 dataset with 400 workers}
   \end{subfigure}
  \caption{The observed delays at server nodes. Left column: the first 3,000
    delays on one server node. Right column: the histogram of all delays.}
\label{fig:delay}
\end{figure}

\begin{figure}[th!]
  \centering

 \begin{subfigure}[b]{.4\textwidth}
   \centering
  \includegraphics[width=\textwidth]{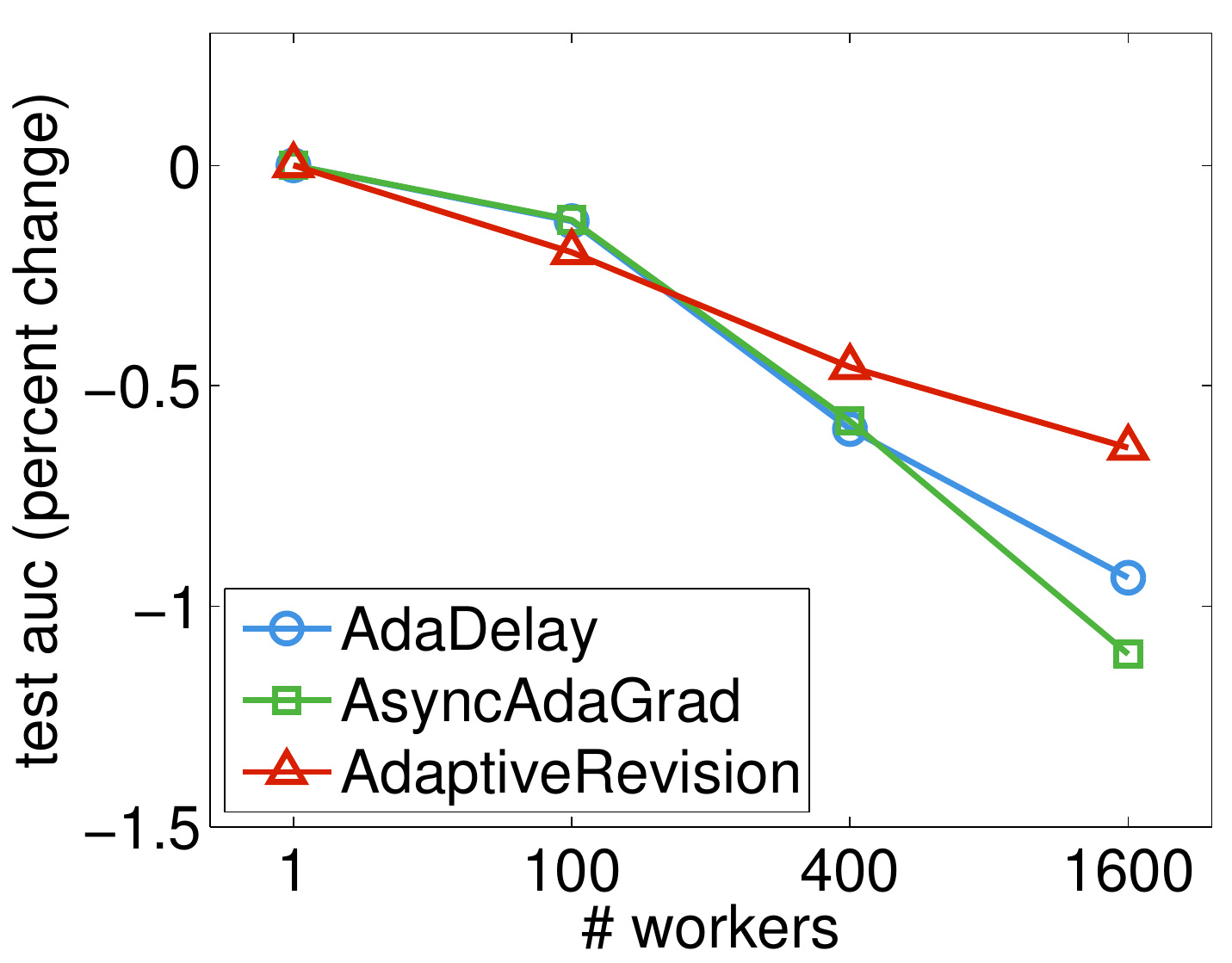}
\caption{Criteo}
   \end{subfigure}\hspace{4ex}%
 \begin{subfigure}[b]{.4\textwidth}
   \centering
  \includegraphics[width=\textwidth]{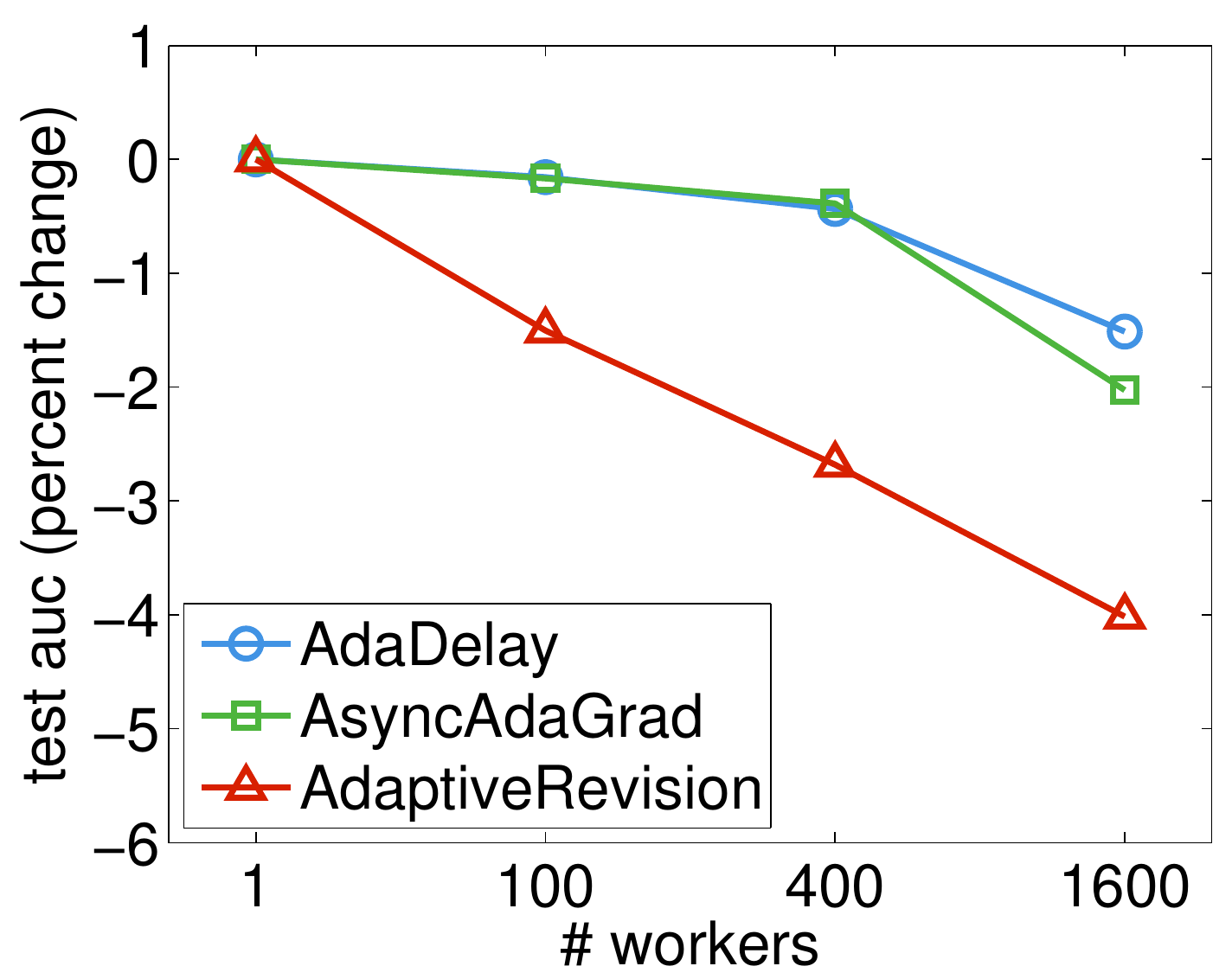}
\caption{CTR2}
   \end{subfigure}
  \caption{Test AUC as function of maximal delays.}
  \label{fig:auc}
\end{figure}

Next, we present the comparison results of these three algorithms by varying the number of workers.  We use
the AUC on the validation dataset as the criterion\footnote{We observed similar
  results on using LogLoss.}, often 1\% difference is significant for
click-through rate estimation.
We set the minibatch size to $10^5$ and $10^4$
for Criteo and CTR2, respectively, to reduce the communication frequency for
better system performance\footnote{Probably due to the scale and the
  sparsity of the datasets, we observed no significant improvement on the test AUC when
  decreasing the minibatch size even to 1.}.
We search $\alpha_0$ in the range
$[10^{-4}, 1]$ and report the best results for each algorithm in Figure~\ref{fig:auc}.

As can be seen, AdaptiveRevision only outperforms
AsyncAdaGrad on the Criteo dataset with a large number of workers. The reason
why it
differs from \cite{mcmahan2014} is probably due to the datasets we used are
1000 times larger than the ones reported by \cite{mcmahan2014}, and we evaluated the
algorithms in a distributed environment rather than a simulated setting where a
large minibatch size is necessary for the former. However, as reported
\cite{mcmahan2014}, we also observed that AdaptiveRevision's best
learning rate is insensitive to the number of workers.

On the other hand, AdaDelay improves AsyncAdaGrad when
a large number of workers (greater than $400$) is used, which means the delay adaptive
learning rate takes effect when the delay can be large.

To further investigate this phenomenon, we simulated an overloaded cluster where
several stragglers may produce large delays; we do this by slowing down half of the workers by a random factor in $\{1,4\}$ when
computing gradients. The results are shown in Figure \ref{fig:auc2}. As can be
seen, AdaDelay consistently outperforms AsyncAdaGrad, which shows that adaptive modeling of the actual delay is better
than using a constant worst case delay when the variance of the delays is large.

\begin{figure}[th!]
  \centering
 \begin{subfigure}[b]{.4\textwidth}
   \centering
  \includegraphics[width=\textwidth]{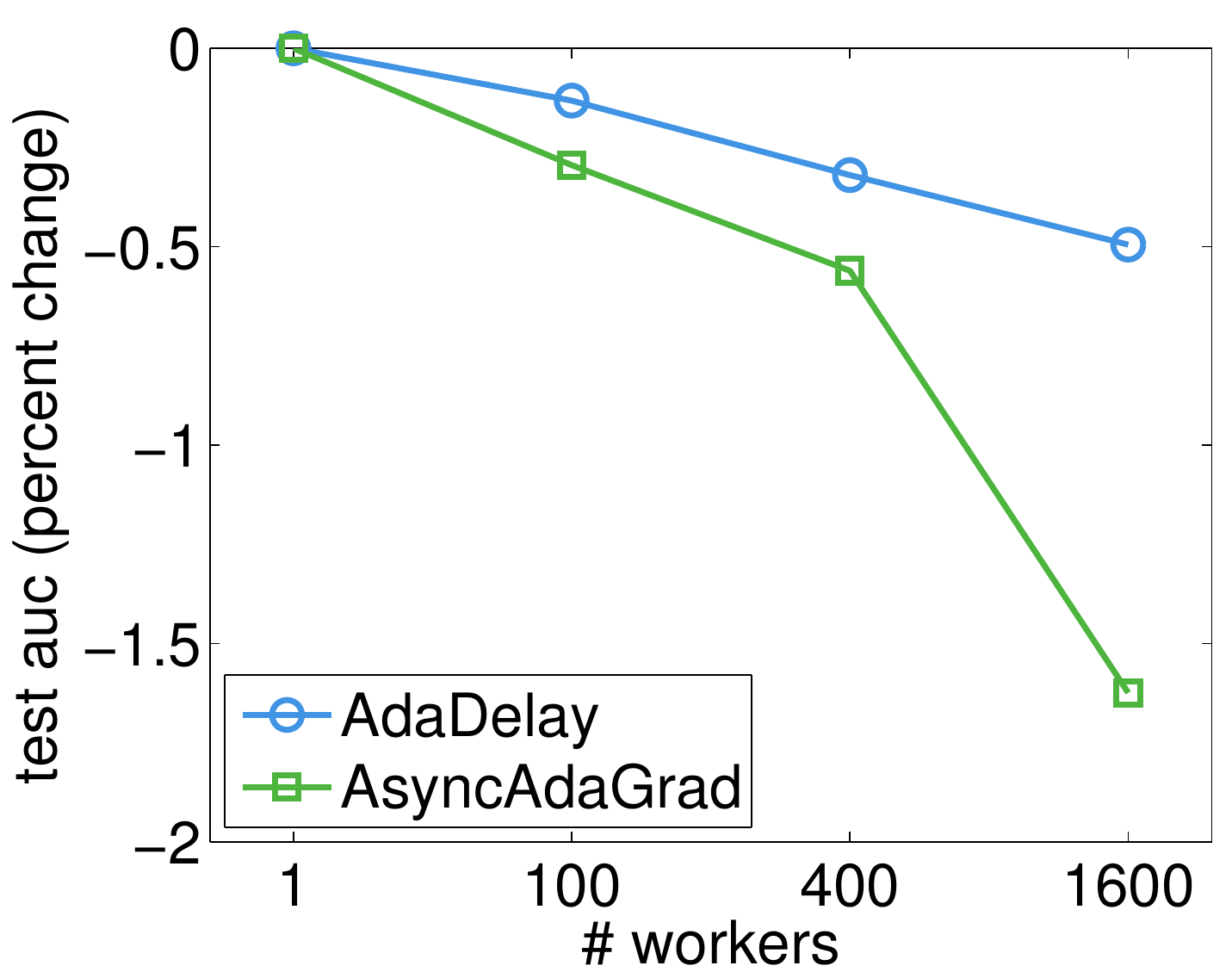}
\caption{Criteo}
   \end{subfigure}\hspace{4ex}%
 \begin{subfigure}[b]{.4\textwidth}
   \centering
  \includegraphics[width=\textwidth]{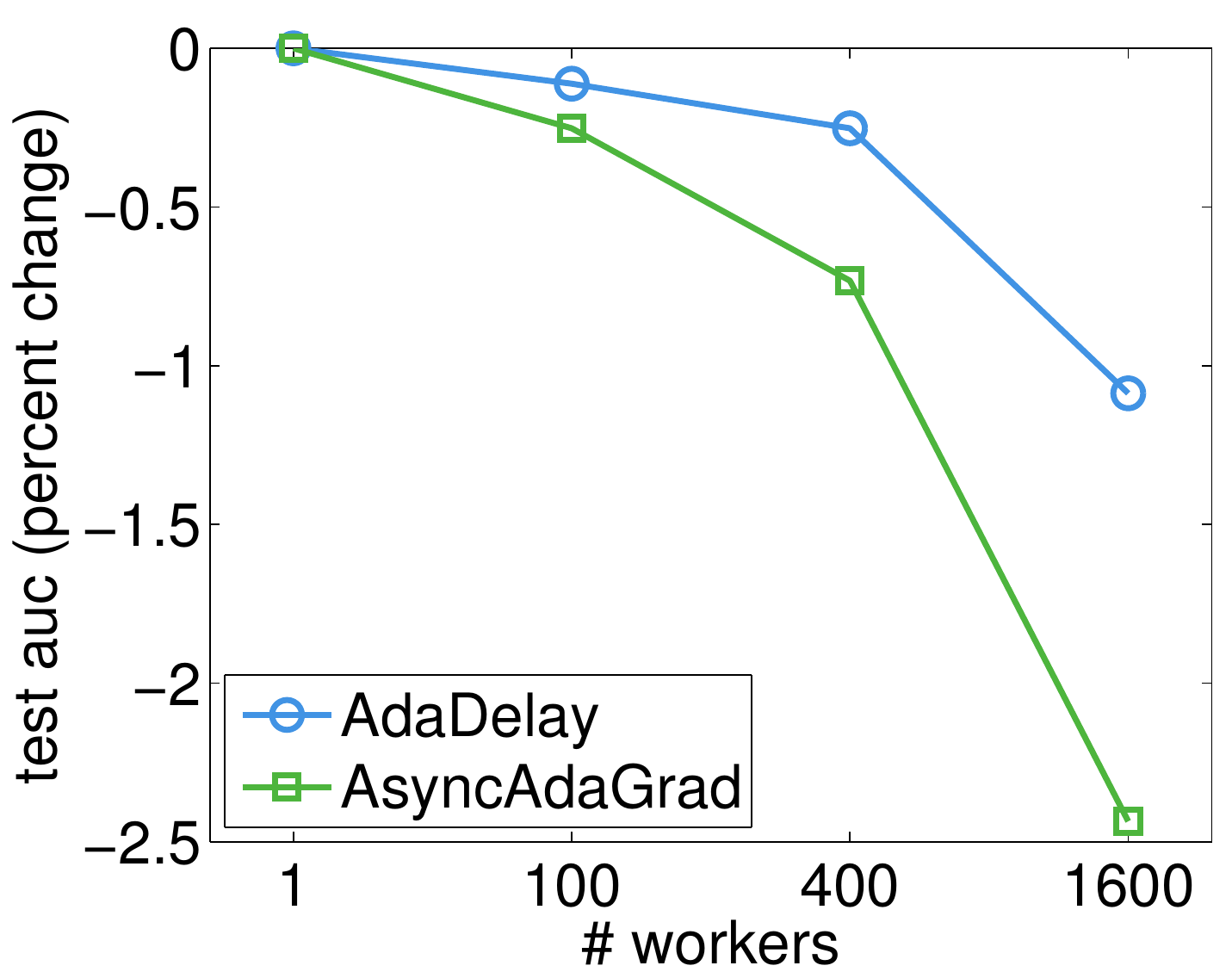}
\caption{CTR2}
   \end{subfigure}
  \caption{Test AUC as function of maximal delays with the exiting of stragglers.}
  \label{fig:auc2}
\end{figure}

Finally we report the system performance.  We first present the speedup from 1 machine to
16 machines, where each machine runs 100 workers.  We observed a near linear
speedup of AdaDelay, which is shown in
Figure~\ref{fig:speedup}. The main reason is due to the asynchronous updating
which removes the dependencies between worker nodes.  In addition, using
multiple workers within a machine can fully utilize the computational resources
by hiding the overhead of reading data and communicating the parameters.

\begin{wrapfigure}{r}{0.5\textwidth}
  \centering
  \includegraphics[width=.4\textwidth]{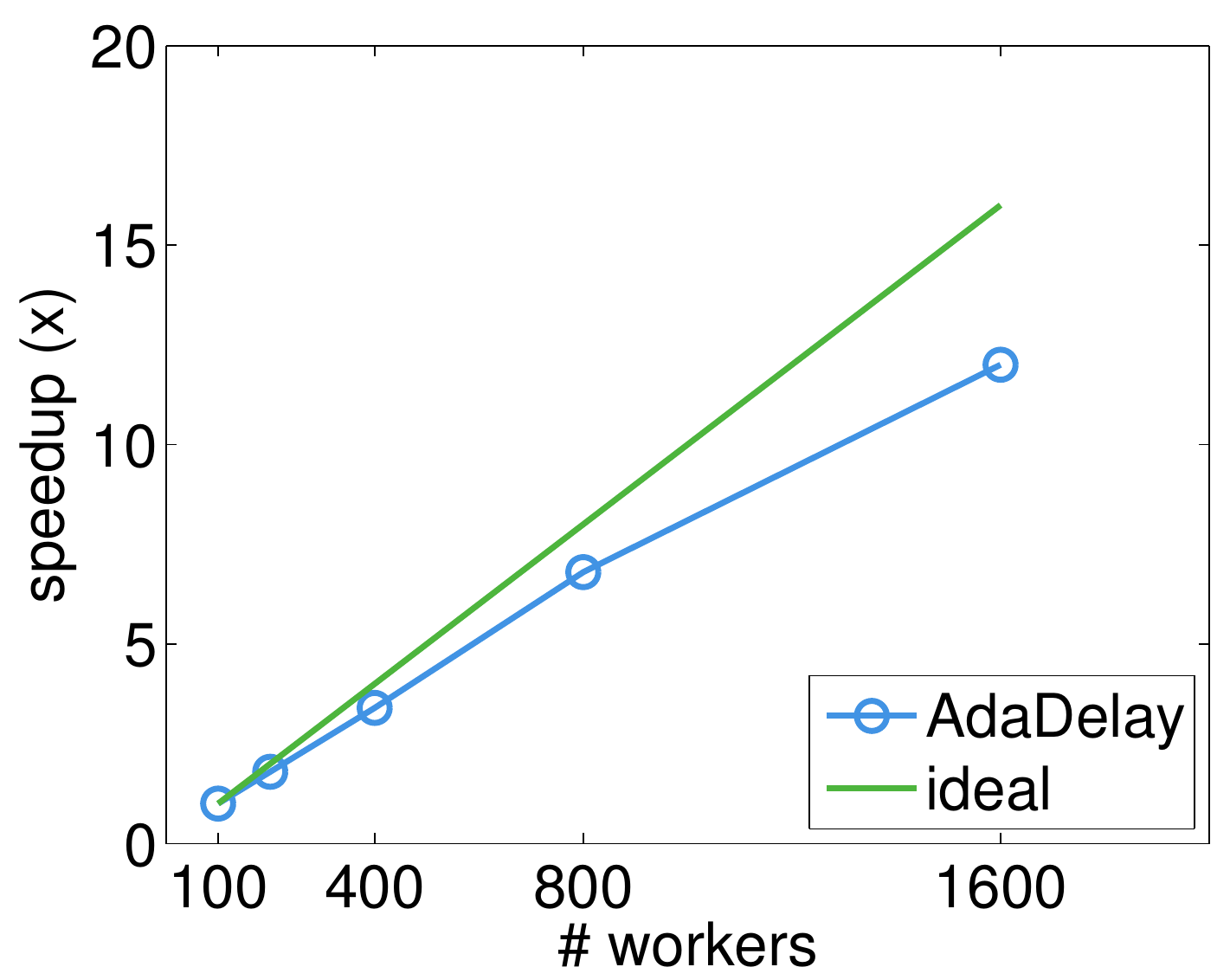}
  \caption{The speedup of AdaDelay.}
  \label{fig:speedup}
  \vskip -3ex
\end{wrapfigure}

In the parameter server framework, worker nodes only need to cache one or a few
data minibatches. Most memory is used by the server nodes to store the model. We
summarize the server memory usage for the three algorithms compared in
Table~\ref{tab:mem}.

As expected, AdaDelay and AsyncAdaGrad have similar memory
consumption because the extra storage needed by AdaDelay to track and compute the incurred delays $\tau_t$ is tiny. However AdaptiveRevision doubles memory usage, because of the extra entries that it needs for each feature, and because of the cached
delayed gradient $g^{\textrm{bak}}$.

\begin{table}[t!]
  \centering
  \begin{tabular}{|l|ccc|}
    \hline
    & AdaDelay & AsyncAdaGrad & AdaptiveRevision \\
    \hline
    Criteo & 24GB & 24 GB & 55 GB\\
    CTR2 & 97 GB & 97 GB & 200 GB\\
    \hline
  \end{tabular}
  \caption{Total memory used by server nodes.}
  \label{tab:mem}
\end{table}

\section{Conclusions}
\label{sec:conslusion}
 In real distributed computing environment, there are multiple factors contributing to delay, such as the CPU speed, I/O of disk, and network throughput. With the inevitable and sometimes unpredictable phenomenon of delay, 
we considered distributed convex optimization by developing and analyzing \textit{AdaDelay}, an asynchronous SGD method that tolerates stale gradients.

 A key component of our work that differs from existing approaches is the use of (server-side) updates sensitive to the actual delay observed in the network. This allows us to use larger stepsizes initially, which can lead to more rapid initial convergence, and stronger ability to adapt to the environment. We discussed details of two different realistic delay models: (i) uniform (more generally, bounded support) delays, and (ii) scaled delays with constant first and second moments but not-necessarily bounded support. Under both models, we obtain theoretically optimal convergence rates.

Adapting more closely to observed delays and incorporating server-side delay sensitive gradient aggregation that combines the benefits of the adaptive revision framework~\citep{mcmahan2014} with our delayed gradient methods is an important future direction. Extension of our analysis to handle constrained convex optimization problems (without requiring a projection oracle onto the constraint set) is also an important part of future work.

\bibliographystyle{plainnat}

\begin{thebibliography}{18}
\providecommand{\natexlab}[1]{#1}
\providecommand{\url}[1]{\texttt{#1}}
\expandafter\ifx\csname urlstyle\endcsname\relax
  \providecommand{\doi}[1]{doi: #1}\else
  \providecommand{\doi}{doi: \begingroup \urlstyle{rm}\Url}\fi

\bibitem[Agarwal and Duchi(2011)]{agDuc11}
Alekh Agarwal and John~C Duchi.
\newblock Distributed delayed stochastic optimization.
\newblock In \emph{Advances in Neural Information Processing Systems}, pages
  873--881, 2011.

\bibitem[Bertsekas and Tsitsiklis(1989)]{BerTsi89}
D.~Bertsekas and J.~Tsitsiklis.
\newblock \emph{Parallel and Distributed Computation: {N}umerical Methods}.
\newblock Prentice-Hall, 1989.

\bibitem[Bertsekas(2011)]{bertsekas2011}
Dimitri~P Bertsekas.
\newblock Incremental gradient, subgradient, and proximal methods for convex
  optimization: A survey.
\newblock \emph{Optimization for Machine Learning}, 2010:\penalty0 1--38, 2011.

\bibitem[Duchi et~al.(2013)Duchi, Jordan, and McMahan]{duchi2013}
John Duchi, Michael~I Jordan, and Brendan McMahan.
\newblock Estimation, optimization, and parallelism when data is sparse.
\newblock In \emph{NIPS 26}, pages 2832--2840, 2013.

\bibitem[Duchi et~al.(2012)Duchi, Agarwal, and Wainwright]{duchi2012dual}
John~C Duchi, Alekh Agarwal, and Martin~J Wainwright.
\newblock Dual averaging for distributed optimization: convergence analysis and
  network scaling.
\newblock \emph{Automatic Control, IEEE Transactions on}, 57\penalty0
  (3):\penalty0 592--606, 2012.

\bibitem[Ghadimi and Lan(2012)]{ghadimi2012}
Saeed Ghadimi and Guanghui Lan.
\newblock Optimal stochastic approximation algorithms for strongly convex
  stochastic composite optimization i: A generic algorithmic framework.
\newblock \emph{SIAM Journal on Optimization}, 22\penalty0 (4):\penalty0
  1469--1492, 2012.

\bibitem[Langford et~al.(2009)Langford, Smola, and Zinkevich]{LanSmoZin09}
J.~Langford, A.~J. Smola, and M.~Zinkevich.
\newblock Slow learners are fast.
\newblock In \emph{Neural Information Processing Systems}, 2009.
\newblock URL \url{http://arxiv.org/abs/0911.0491}.

\bibitem[Li et~al.(2014{\natexlab{a}})Li, Andersen, Park, Smola, Ahmed,
  Josifovski, Long, Shekita, and Su]{li2014}
Mu~Li, David~G Andersen, Jun~Woo Park, Alexander~J Smola, Amr Ahmed, Vanja
  Josifovski, James Long, Eugene~J Shekita, and Bor-Yiing Su.
\newblock Scaling distributed machine learning with the parameter server.
\newblock In \emph{Operating Systems Design and Implementation (OSDI)},
  2014{\natexlab{a}}.

\bibitem[Li et~al.(2014{\natexlab{b}})Li, Andersen, Smola, and Yu]{li2014b}
Mu~Li, David~G Andersen, Alex~J Smola, and Kai Yu.
\newblock Communication efficient distributed machine learning with the
  parameter server.
\newblock In \emph{NIPS 27}, pages 19--27, 2014{\natexlab{b}}.

\bibitem[McMahan and Streeter(2014)]{mcmahan2014}
Brendan McMahan and Matthew Streeter.
\newblock Delay-tolerant algorithms for asynchronous distributed online
  learning.
\newblock In \emph{NIPS 27}, pages 2915--2923, 2014.

\bibitem[Nedi{\'c} et~al.(2001)Nedi{\'c}, Bertsekas, and
  Borkar]{nedic2001distributed}
A~Nedi{\'c}, Dimitri~P Bertsekas, and Vivek~S Borkar.
\newblock Distributed asynchronous incremental subgradient methods.
\newblock \emph{Studies in Computational Mathematics}, 8:\penalty0 381--407,
  2001.

\bibitem[Nemirovski et~al.(2009)Nemirovski, Juditsky, Lan, and
  Shapiro]{nemirov09}
A.~Nemirovski, A.~Juditsky, G.~Lan, and A.~Shapiro.
\newblock Robust stochastic approximation approach to stochastic programming.
\newblock \emph{SIAM Journal on Optimization}, 19\penalty0 (4):\penalty0
  1574--1609, 2009.

\bibitem[Ram et~al.(2010)Ram, Nedi{\'c}, and Veeravalli]{ram2010}
S~Sundhar Ram, A~Nedi{\'c}, and Venugopal~V Veeravalli.
\newblock Distributed stochastic subgradient projection algorithms for convex
  optimization.
\newblock \emph{Journal of optimization theory and applications}, 147\penalty0
  (3):\penalty0 516--545, 2010.

\bibitem[Robbins and Monro(1951)]{RobMon51}
H.~Robbins and S.~Monro.
\newblock A stochastic approximation method.
\newblock \emph{Annals of Mathematical Statistics}, 22:\penalty0 400--407,
  1951.

\bibitem[Schmidt et~al.(2013)Schmidt, Roux, and Bach]{Schmidt13}
Mark~W. Schmidt, Nicolas~Le Roux, and Francis~R. Bach.
\newblock Minimizing finite sums with the stochastic average gradient.
\newblock \emph{CoRR}, abs/1309.2388, 2013.

\bibitem[Shamir and Srebro(2014)]{shamir2014}
Ohad Shamir and Nathan Srebro.
\newblock Distributed stochastic optimization and learning.
\newblock In \emph{Proceedings of the 52nd Annual Allerton Conference on
  Communication, Control, and Computing}, 2014.

\bibitem[Shapiro et~al.(2014)Shapiro, Dentcheva, and Ruszczy{\'n}ski]{shapiro}
Alexander Shapiro, Darinka Dentcheva, and Andrzej Ruszczy{\'n}ski.
\newblock \emph{Lectures on stochastic programming: modeling and theory},
  volume~16.
\newblock SIAM, 2014.

\bibitem[Srebro and Tewari(2010)]{sreTew10}
N.~Srebro and A.~Tewari.
\newblock {Stochastic Optimization for Machine Learning}.
\newblock ICML 2010 Tutorial, 2010.

\end{thebibliography}

\begin{appendices}
\section{Technical details of the convergence analysis}
We collect below some basic tools and definitions from convex analysis.
\begin{defn}[Bregman divergence]\label{def:bregman}
  Let $h: \Xc \times \Xc \to [0,\infty]$ be differentiable strictly convex function. The \emph{Bregman divergence} generated by $h$ is
  \begin{equation}
    \label{eq:2}
    D_h(x,y) := h(x) - h(y) - \ip{\nabla h(y)}{x-y},\qquad x, y \in \Xc.
  \end{equation}
\end{defn}

\begin{list}{--}{\leftmargin=2em}
\item \textbf{Fenchel conjugate:}
\begin{equation}
\label{eq:fenchel_con}
f^*(y) = \sup_{x\in \Xc}\langle x,y \rangle - f(x)
\end{equation}
\item \textbf{Prox operator:}
\begin{equation}
\label{eq:prox}
\text{prox}_f(x) = \argmin_{y\in \Xc} f(y) + {1\over 2} \|x-y\|_2^2, \qquad\forall\ x \in \Xc
\end{equation}
\item \textbf{Moreau decomposition:}
\begin{equation}
 \label{eq:moreau}
x = \text{prox}_f(x) + \text{prox}_{f^*}(x), \qquad\forall\ x \in \Xc
\end{equation}
\item \textbf{Fenchel-Young inequality:}
  \begin{equation}
    \label{eq:fench}
    \ip{x}{y} \le f(x) + f^*(y)
  \end{equation}
\item \textbf{Projection lemma:}
  \begin{equation}
    \label{eq:proj}
    \ip{y-\Pi_{\Xc}(y)}{x-\Pi_{\Xc}(y)} \le 0,\qquad\forall\ x \in \Xc.
  \end{equation}
\item \textbf{Descent lemma:}
  \begin{equation}
    \label{eq:descent}
    f(y) \le f(x) + \ip{\nabla f(x)}{y-x} + \tfrac{L}{2}\norm{y-x}^2.
  \end{equation}
\item \textbf{Four-point identity:}
  Bregman divergences satisfy the following \emph{four point identity}:
  \begin{equation}
    \label{eq:four}
    \ip{\nabla h(a)-\nabla h(b)}{c-d} = D_h(d,a) - D_h(d,b)-D_h(c,a)+D_h(c,b).
  \end{equation}
  A special case of~\eqref{eq:four} is the ``three-point'' identity
  \begin{equation}
    \label{eq:three}
    \ip{\nabla h(a)-\nabla h(b)}{b-c} = D_h(c,a)- D_h(c,b)-D_h(b,a).
  \end{equation}
\end{list}

\subsection{Bounding the change $f(x(t+1))-f(x^*)$}
We start the analysis by bounding the gap  $f(x(t+1) )-f(x^*)$. The lemma below is just a combination of several results of~\citep{agDuc11}. We present the details below in one place for easy reference. The impact of our delay sensitive step sizes shows up in subsequent lemmas, where we bound the individual terms that arise from Lemma~\ref{lem:e}.
\begin{lemma}\label{lem:e}
  At any time-point $t$, let the gradient error due to delays be
  \begin{equation}
    \label{eq:16}
    e_t := \nabla f(x(t)) - g(t-\tau_t).
  \end{equation}
  Then, we have the following (deterministic) bound:
  \begin{align}
    \nonumber
    &f(x(t+1) )-f(x^*)\\
    \nonumber
    &=\frac{1}{2\alpha(t,\tau_t)}\left[\norm{x^*-x(t)}^2-\norm{x^*-x(t+1) }^2\right] + 
    \ip{e_t}{x(t+1) -x^*} + \tfrac{L-1/\alpha(t,\tau_t)}{2}\norm{x(t)-x(t+1) }^2,\\
    \nonumber
    &\le\frac{1}{2\alpha(t,\tau_t)}\left[\norm{x^*-x(t)}^2-\norm{x^*-x(t+1) }^2\right] + 
    \ip{\nabla f(x(t))-\nabla f(x(t-\tau_t))}{x(t+1) -x^*}\\
    \label{eq:17}
    &\qquad+ \ip{\nabla f(x(t-\tau_t))-g(t-\tau_t)}{x(t)-x^*}
    + \tfrac{1}{2\eta(t,\tau_t)}\norm{\nabla f(x(t-\tau_t))-g(t-\tau_t)}^2.
  \end{align}
\end{lemma}
\begin{proof}
  Using convexity of $f$ we have
  \begin{equation}
    \label{eq:13}
    \begin{split}
      f(x_{t})-f(x^*) &\le \ip{\nabla f(x(t))}{x(t+1) -x^*} + \ip{\nabla f(x(t))}{x(t)-x(t+1) }.
    \end{split}
  \end{equation}
  Now apply Lipschitz continuity of $\nabla f$ to the second term to obtain
  \begin{equation}
    \label{eq:14}
    \begin{split}
      f(x_{t})-f(x^*) &\le \ip{\nabla f(x(t))}{x(t+1) -x^*} + f(x(t))-f(x(t+1) ) 
      + \tfrac{L}{2}\norm{x(t)-x(t+1) }^2,\\
      \implies f(x(t+1) ) -f(x^*) &\le \ip{\nabla f(x(t))}{x(t+1) -x^*} + \tfrac{L}{2}\norm{x(t)-x(t+1) }^2.
    \end{split}
  \end{equation}
  Using the definition~\eqref{eq:16} of the gradient error $e_t$, we can rewrite~\eqref{eq:14} as
  \begin{align*}
    f(x(t+1) ) -f(x^*) &\le \underbrace{\ip{g(t-\tau_t)}{x(t+1) -x^*}}_{T1} + \underbrace{\ip{e_t}{x(t+1) -x^*}}_{T2} + \tfrac{L}{2}\norm{x(t)-x(t+1) }^2.
  \end{align*}
  To complete the proof, we bound the terms $T1$ and $T2$ separately below.

  \emph{Bounding T1:}
  Since $x(t+1) $ is a minimizer in~\eqref{eq:iter1}, from the projection inequality~\eqref{eq:proj} we have
  \begin{equation*}
    \ip{x(t)-\alpha(t,\tau_t)g(t-\tau_t)-x(t+1) }{x-x(t+1) } \le 0,\quad\forall x \in \Xc.
  \end{equation*}
  Choose $x=x^*$; then rewrite the above inequality and identity~\eqref{eq:three} with $h(x)={1\over 2} \|x\|^2$ to get
  \begin{equation*}
    \begin{split}
      \alpha(t,\tau_t)\ip{g(t-\tau_t)}{x(t+1) -x^*} &\le \ip{x(t)-x(t+1) }{x(t+1) -x^*}\\
      &=\half\norm{x^*-x(t)}^2-\half\norm{x^*-x(t+1) }^2-\half\norm{x(t+1) -x(t)}^2;
    \end{split}
  \end{equation*}
  Plugging in this bound for $T1$ and collecting the $\norm{x(t+1) -x(t)}^2$ terms we obtain
  \begin{align}
    \nonumber
    &f(x(t+1) )-f(x^*)\\
    \nonumber
    &\le \tfrac{1}{2\alpha(t,\tau_t)} \left[\norm{x^*-x(t)}^2-\norm{x^*-x(t+1) }^2 - \norm{x(t+1) -x(t)}^2\right] 
    + \ip{e_t}{x(t+1) -x^*} + \tfrac{L}{2}\norm{x(t)-x(t+1) }^2\\
    \label{eq:3}
    &=\tfrac{1}{2\alpha(t,\tau_t)}\left[\norm{x^*-x(t)}^2-\norm{x^*-x(t+1) }^2\right] + 
    \ip{e_t}{x(t+1) -x^*} + \tfrac{L-1/\alpha(t,\tau_t)}{2}\norm{x(t)-x(t+1) }^2.
  \end{align}

  \emph{Bounding T2:} Adding and subtracting $\nabla f(x(t-\tau_t))$ we obtain
  \begin{align*}
    &\ip{e_t}{x(t+1) -x^*} = \ip{\nabla f(x(t))-g(t-\tau_t)}{x(t+1) -x^*}\\
    &=\ip{\nabla f(x(t))-\nabla f(x(t-\tau_t))}{x(t+1) -x^*} + \ip{\nabla f(x(t-\tau_t))-g(t-\tau_t)}{x(t+1) -x^*}\\
    &=\ip{\nabla f(x(t))-\nabla f(x(t-\tau_t))}{x(t+1) -x^*} + \ip{\nabla f(x(t-\tau_t))-g(t-\tau_t)}{x_{t}-x^*}\\
    &\qquad+ \ip{\nabla f(x(t-\tau_t))-g(t-\tau_t)}{x(t+1) -x(t)}\\
    &\le\ip{\nabla f(x(t))-\nabla f(x(t-\tau_t))}{x(t+1) -x^*} + \ip{\nabla f(x(t-\tau_t))-g(t-\tau_t)}{x(t)-x^*}\\
    &\qquad+\tfrac{1}{2\eta(t,\tau_t)}\norm{\nabla f(x(t-\tau_t))-g(t-\tau_t)}^2 + \tfrac{\eta(t,\tau_t)}{2}\norm{x(t+1) -x(t)}^2,
  \end{align*}
  where the last inequality is an application of~\eqref{eq:fench}. Adding this inequality to~\eqref{eq:3} and using $1/\alpha(t,\tau_t)=L+\eta(t,\tau_t)$, we obtain~(\ref{eq:17}).
\end{proof}

The next step is to take expectations over~\eqref{eq:17} and then further bound the resulting  terms separately. Note that $\nabla f(x(t-\tau_t))-g(t-\tau_t)$ is independent of $x(t)$ given $g(1),\ldots,g(t-\tau_t-1)$ (since $x(t)$ is a function of gradients up to time $t-\tau_t-1$). Thus, the third term in~\eqref{eq:17} has zero expectation. It remains to consider expectations over the following three quantities:
\begin{align}
  \label{eq:7}
  \Delta(t)  &:= \frac{1}{2\alpha(t,\tau_t)}\left[\norm{x^*-x(t)}^2-\norm{x^*-x(t+1) }^2\right];\\
  \label{eq:8}
  \Gamma(t)  &:= \ip{\nabla f(x(t))-\nabla f(x(t-\tau_t))}{x(t+1) -x^*};\\
  \label{eq:9}
  \Sigma(t)  &:= \tfrac{1}{2\eta(t,\tau_t)}\norm{\nabla f(x(t-\tau_t))-g(t-\tau_t)}^2.
\end{align}
Lemma~\ref{lem:delta1} bounds \eqref{eq:7} under Assumption~\ref{ass:delay}(A), while Lemma~\ref{lem:delta2} provides a bound under the Assumption~\ref{ass:delay}(B). Similarly, Lemmas~\ref{lem:gamma1} and~\ref{lem:gamma2} bound~\eqref{eq:8}, while Lemmas~\ref{lem:sigma}  bounds~\eqref{eq:9}. Combining these bounds we obtain the theorem.
 
\subsection{Bounding $\Delta$, $\Gamma$, and $\Sigma$}
% Lemma delta1
\begin{lemma}
  Let $\Delta(t)$ be given by~\eqref{eq:7}, and let Assumption~\ref{ass:delay}~(A) hold. Then,
  \label{lem:delta1}
    \begin{equation*}
    \sum_{t=1}^T\E[\Delta(t)]
    ={1\over 2} \sum_{t=1}^T\E\bigl[\frac{1}{\alpha(t,\tau_t)}\left(\norm{x^*-x(t)}^2 - \norm{x^*-x(t+1) }^2\right)\bigr]
    \le \half (L+c)R^2 + \sqrt{2}cR^2\btau\sqrt{T}.
  \end{equation*}
\end{lemma}
\begin{proof}
  Unlike the delay independent step sizes treated in~\citep{agDuc11}, bounding $\Delta(t)$ requires some more work because $\alpha(t,\tau_t)$ depends on $\tau_t$, which in turn breaks the monotonically decreasing nature of $\alpha(t,\tau_t)$ (we wish to avoid using a fixed worst case bound on the steps, to gain more precise insight into the impacts of being sensitive to delays), necessitating a more intricate analysis. 

  Let $r_t = \norm{x(t)-x^*}^2$. Observe that although $r_t \Perp \tau_t$, it is \emph{not} independent of $\tau(t-1)$. Thus, with 
  $$z_t = \frac{1}{\alpha(t,\tau_t)}-\frac{1}{\alpha(t-1,\tau_{t-1})}
  = c(\sqrt{t+\tau_t} - \sqrt{t-1+\tau_{t-1}}),$$ 
  we have
  \begin{align}
    \label{eq:21}
    \sum_{t=1}^T\E[\Delta(t)]&=\frac12\E\Bigl[\frac{r_1}{\alpha(1,\tau(1))} + \sum_{t=2}^Tz_tr_t\Bigr] \le \frac12(L+c)R^2 + \frac12\E\Bigl[\sum_{t=2}^Tz_tr_t\Bigr].
  \end{align}
  Since $\alpha(t,\tau_t)$ is \emph{not} monotonically decreasing with $t$, while upper-bounding $\E[\Delta(t)]$ we cannot simply discard the final term in~\eqref{eq:21}. 

  When $\tau(t-1) \sim U(\set{0,2\btau})$, $r_t$ uniformly takes on at most $2\btau+1$ values
  \begin{equation*}
    r_{t,s} := \norm{x_{t,s}-x^*}^2,\qquad s \in [2\btau],
  \end{equation*}
  where $x_{t,s} = \Pi_{\Xc}[x_{t-1}-\alpha(t-1,\tau(t-1)=s)g(t-1,\tau(t-1))]$.  Given a delay $\tau(t-1)=s$, $r_t$ is just $r_{t,s}$. Using $z_t = \alpha(t)^{-1}-\alpha(t-1)^{-1} = c\sqrt{t+\tau_t}-c\sqrt{t-1+\tau_{t-1}}$, we have
  \begin{equation*}
    z_{t,s} = c\left(\sqrt{t+\tau_t} - \sqrt{t-1+s}\right),\qquad s \in [2\btau].
  \end{equation*}
  Using nested expectations $\E[z_tr_t] = \E_{\tau_t}[ \E[z_tr_t|\tau_t]]$ we then see that
  \begin{align*}
    \E[z_tr_t] &= \frac{1}{2\btau+1}\sum_{l=0}^{2\btau}\left(\sum_{s=0}^{2\btau}(2\btau+1)^{-1}r_{t,s}c\left(\sqrt{t+l} - \sqrt{t-1+s}\right) \right)\\
    &\le \frac{1}{2\btau+1}\sum_{l=0}^{2\btau}\left(\sum_{s=0}^{l-1}(2\btau+1)^{-1}r_{t,s}c\left(\sqrt{t+l} - \sqrt{t-1+s}\right) \right),
  \end{align*}
  where we dropped the terms with $s \ge l$ as they are non-positive.

  Consider now the inner summation above. We have
  \begin{align*}
    &\frac{c}{2\btau+1}\sum_{s=0}^{l-1}r_{t,s}\left(\sqrt{t+l}-\sqrt{t-1+s}\right)\\
    &\le\frac{cR^2}{2\btau+1}\sum_{s=0}^{l-1}\left(\sqrt{t+l}-\sqrt{t-1+s}\right)\\
    &=\frac{cR^2}{2\btau+1}\sum_{s=0}^{l-1}\frac{l-s+1}{\sqrt{t+l}+\sqrt{t-1+s}}\\
%    &\le\frac{cR^2}{2\btau+1}\sum_{s=0}^{l-1}\frac{l-s+1}{\sqrt{2t+l+s-1}}\\
    &\le\frac{cR^2}{2\btau+1}\frac{1}{\sqrt{2t-1}}\sum_{s=0}^{l-1}(l-s+1)\\
    &=\frac{cR^2}{2\btau+1}\frac{1}{\sqrt{2t-1}}{3l+l^2\over 2}.
  \end{align*}
  Thus, we now consider
  \begin{align*}
    \E[z_tr_t] &\le \frac{1}{2\btau+1}\sum_{l=0}^{2\btau}\frac{cR^2}{2\btau+1}\frac{1}{\sqrt{2t-1}}{3l+l^2\over 2}\\
    &=\frac{cR^2}{(2\btau+1)^2\sqrt{2t-1}}(2\btau+1)(4\btau+2.5)\btau\\
    &<\frac{2cR^2\btau}{\sqrt{2t-1}}.
  \end{align*}
  Summing over $t=2$ to $T$, we finally obtain the upper bound
  \begin{equation*}
    \sum_{t=2}^T\E[z_tr_t] \le cR^2\btau\sum_{t=2}^T\frac{1}{\sqrt{2t-1}}
    \le 2cR^2\btau\sqrt{2T}.\qedhere
  \end{equation*}
\end{proof}

\begin{lemma}
  \label{lem:delta2}
 Let Assumption \eqref{ass:delay} (B) hold. Then  
  \begin{equation*}
    \sum_{t=1}^T\E[\Delta(t)]
    \le {1\over 2} R^2(L+c) + 
    {1\over 2} cR^2 \sum_{t=2}^{T} \frac{\btau_t+1}{\sqrt{2t-1}}.
%    \sum_{t=2}^T\frac{\btau_t+1}{\sqrt{2t-1}}.
  \end{equation*}
\end{lemma}
\begin{proof}
  Proceeding as for Lemma~\ref{lem:delta1}, according to \eqref{eq:21}, the task reduces to bounding $\E[z_tr_t]$. Consider thus,
  \begin{equation*}
    \E[z_tr_t] \le \E[z_t^+r_t] \le R^2\E[z_t^+],
  \end{equation*}
  where we use $z_t^+$ to denote $\max(z_t, 0)$.
  Let us now control the last expectation. Let $P_t(l)=\mathbb{P}(\tau(t)=l)$, then
  \begin{align*}
    \E[z_t^+] &= \sum_{\tau_t,\tau_{t-1}}P(\tau_t,\tau_{t-1})\max(0,z_t)\\
    &=c\sum_{l=0}^{t-1}\sum_{s=0}^{t-2}P_t(l)P_{t-1}(s)[\sqrt{t+l}-\sqrt{t-1+s}]^+\\
%    &=c\sum_{l=0}^{t-1}\sum_{s=0}^{l}P_t(l)P_{t-1}(s)\bigl(\sqrt{t+l}-\sqrt{t-1+s}\bigr)\\
    &=c\sum_{l=0}^{t-1}\sum_{s=0}^{l}P_t(l)P_{t-1}(s)\frac{l+1-s}{\sqrt{t+l}+\sqrt{t-1+s}}\\
%    &\le c\sum_{l=0}^{t-1}\sum_{s=0}^l P_t(l)P_{t-1}(s)\frac{l+1-s}{\sqrt{2t+l+s-1}}\\
    &\le c\sum_{l=0}^{t-1}\sum_{s=0}^lP_t(l)P_{t-1}(s)\frac{l+1}{\sqrt{2t+l-1}}\\
    %&=c\sum_{l=0}^{t-1}P_t(l)\frac{l+1}{\sqrt{2t+l-1}}\left[1-\sum_{s=l+1}^{t-2}P(s)\right]\\
    &\le c\sum_{l=0}^{t-1}P_t(l)\frac{l+1}{\sqrt{2t+l-1}}\\
    &\le c\sum_{l=0}^{t-1}P_t(l)\frac{l+1}{\sqrt{2t-1}} = c\frac{\btau_t+1}{\sqrt{2t-1}}.
  \end{align*}
So 
\begin{equation*}\sum_{t=2}^T R^2 \E[z_t^+]\le cR^2 \sum_{t=2}^{T} \frac{\btau_t+1}{\sqrt{2t-1}}. \qedhere
\end{equation*}  
\end{proof}
% Lemma gamma1
\begin{lemma}
  \label{lem:gamma1}
  \begin{equation*}
    \begin{split}
      \sum_{t=1}^T \E[\Gamma(t)] &= \sum_{t=1}^T\E\left[\ip{\nabla f(x(t))-\nabla f(x(t-\tau_t))}{x(t+1) -x^*}\right]\\
      &\le  \btau GR + {LC_1\over 2} + {LC_2\over 2}\log T
    \end{split}
  \end{equation*}
    where 
    \begin{equation*}
    C_1={G^2 \btau(\btau+1)(2\btau+1)^2\over 3(L^2+c^2)} ~~ \text{and} ~~ C_2=\frac{G^2(4\btau+3)(\btau+1)}{3c^2}
    \end{equation*}
\end{lemma}
\begin{proof}
  This proof is an adaptation of Lemma 4 and Corollary 1 of~\citet{agDuc11}. First, we exploit convexity of $f$ to help analyze the gradient differences using the four-point identity~\eqref{eq:four}: 
  \begin{equation}
    \label{eq:10}
    \begin{split}
      &\ip{\nabla f(x(t))-\nabla f(x(t-\tau_t))}{x(t+1) -x^*}\\
      &=D_f(x^*,x(t))-D_f(x^*,x(t-\tau_t)) - D_f(x(t+1) ,x(t)) + D_f(x(t+1) ,x(t-\tau_t)).
    \end{split}
  \end{equation}
  Since $\nabla f$ is $L$-Lipschitz, we further have
  \begin{equation*}
    f(x(t+1) ) \le f(x(t-\tau_t)) + \ip{\nabla f(x(t-\tau_t))}{x(t+1) -x(t-\tau_t)} + \tfrac{L}{2}\norm{x(t-\tau_t)-x(t+1) }^2.
  \end{equation*}
  By definition of a Bregman divergence, we also have
  \begin{equation*}
    D_f(x(t+1) ,x(t-\tau_t)) = f(x(t+1) )-f(x(t-\tau_t))
    -\ip{\nabla f(x(t-\tau_t))}{x(t+1) -x(t-\tau_t)},
  \end{equation*}
  which, upon using using \ref{eq:descent}, immediately yields the bound
  \begin{equation*}
    D_f(x(t+1) ,x(t-\tau_t)) \le \tfrac{L}{2}\norm{x(t-\tau_t)-x(t+1) }^2.
  \end{equation*}

  Dropping the negative term $D_f(x(t+1) ,x(t))$ from~\eqref{eq:10} and summing over $t$, we then obtain
  \begin{equation*}
    \begin{split}
      &\sum_{t=1}^T \ip{\nabla f(x(t))-\nabla f(x(t-\tau_t))}{x(t+1) -x^*} \\
      &\le
      \sum_{t=1}^T\left[D_f(x^*,x(t))-D_f(x^*,x(t-\tau_t)) \right] + \frac{L}{2}\sum_{t=1}^T\norm{x(t+1) - x(t-\tau_t)}^2.
    \end{split}
  \end{equation*}
  Notice that the first sum partially telescopes, leaving only the terms not received by the server within the first $T$ iterations. Thus, we obtain the bound
  \begin{equation}
    \label{eq:12}
    \sum_{t: t + \tau_t > T}D_f(x^*,x(t)) + \frac{L}{2}\sum_{t=1}^T\norm{x(t+1) - x(t-\tau_t)}^2.
  \end{equation}
  We bound both each of the terms in~\eqref{eq:12} in turn below. 

  To bound the contribution of the first term in expectation, compute the expected cardinality 
  \begin{equation}
    \label{eq:22}
    \E[|\set{t: t + \tau_t > T}|] = \sum_{t=1}^T\text{Pr}(\tau_t>T-t),
  \end{equation}
  Assuming delays uniform on $\set{0,2\btau}$ bounding this cardinality is easy, since 
  \begin{equation*}
    \text{Pr}(\tau_t > T-t) =
    \begin{cases}
      0 & T-t > 2\btau ,\\
      \frac{2\btau -T+t}{2\btau +1} & \text{otherwise}.
    \end{cases}
  \end{equation*}
  Assuming that $2\btau+1 < T$, \eqref{eq:22} becomes (unsurprisingly)
  \begin{equation*}
    \sum_{s=1}^{2\btau} \frac{2\btau -s}{2\btau +1} = \frac{(4\btau -2\btau)(2\btau+1)}{2(2\btau +1)} = \btau.
  \end{equation*}
  From definition of a Bregman divergence we immediately see that
  \begin{equation*}
    0 \le D_f(x^*,x(t)) \le -\ip{\nabla f(x(t))}{x^*-x(t)} \le \norm{\nabla f(x(t))}\norm{x^*-x(t)} \le GR.
  \end{equation*}
  Thus, the contribution of the first term in~\eqref{eq:12} is bounded in expectation by by $\btau GR$.

  To bound the contribution of the second term, use convexity of $\norm{\cdot}^2$ to obtain
  \begin{align*}
     &\norm{x(t+1)-x(t-\tau_t)}\\
    =& \norm{x(t+1) - x(t) + x(t) - x(t-1) + \cdots + x(t-\tau_t+1)-x(t-\tau_t)}^2\\
    \le& (\tau_t+1)^2\sum_{s=0}^{\tau_t}\tfrac{1}{\tau_t+1}\norm{x_{t+1-s}-x_{t-s}}^2\\
    =& (\tau_t+1)\sum_{s=0}^{\tau_t}
    \norm{\Pi_{\Xc}\bigl(x(t-s)-\alpha(t-s,\tau_{t-s})g(t-s,\tau_{t-s})\bigr)
      - \Pi_{\Xc}(x(t-s))}^2\\
    \le&(\tau_t+1)G^2\sum_{s=0}^{\tau_t}\alpha(t-s,\tau_{t-s})^2.
  \end{align*}
  Conditioned on the delay $\tau_t$ we have
  \begin{equation*}
    \begin{split}
      \E[\norm{x(t+1)-x(t-\tau_t)}^2 | \tau_t] &\le (\tau_t+1)G^2\nlsum_{s=0}^{\tau_t}\E[\alpha(t-s,\tau_{t-s})^2].
  \end{split}
  \end{equation*}
  Under the uniform or scaled assumptions on delays, we obtain similar bounds on the above quantity.  
  
  Consider now the expectation 
  \begin{align*}
    \E[\alpha(t-s,\tau(t-s))^2] &= \E[\frac{1}{L^2+c^2((t-s)+\tau(t-s))+ 2Lc\sqrt{t-s+\tau(t-s)}}] \le \frac{1}{L^2+c^2(t-s)}\\
    &\implies ~\text{if}~ \tau_t=l,~ \sum_{s=0}^{\tau_t}\E[\alpha(t-s,\tau_{t-s})^2] \le \sum_{s=0}^{l}\frac{1}{L^2+c^2(t-l)} = \frac{l+1}{L^2+c^2(t-l)}.
  \end{align*}
  Thus, for $t > 2\btau$, we have the following bound
%   \note{observe that here we implicitly have that , which corresponds to the network having run sufficiently long}\todo{CHECK!}
  \begin{align*}
    \E[\norm{x(t+1)-x(t-\tau_t)}^2] 
    &\le G^2\sum_{l=0}^{2\btau}\frac{1}{2\btau+1}\frac{(l+1)^2}{L^2+c^2(t-l)}\\
    &\le \frac{G^2}{(2\btau+1)(L^2+c^2(t-2\btau))}\sum_{l=0}^{2\btau}(l+1)^2\\
    &= \frac{G^2(4\btau+3)(\btau+1)}{3(L^2+c^2(t-2\btau))}.
  \end{align*}
and for $t\le 2\btau$, we have
  \begin{align*}
    \E[\norm{x(t+1)-x(t-\tau_t)}^2] 
    &\le G^2\sum_{l=0}^{t-1}P_t(l)\frac{(l+1)^2}{L^2+c^2(t-l)}\\
    &\le
    G^2\sum_{l=0}^{t-1}\frac{(l+1)^2}{L^2+c^2}\\
    &={G^2 t(t+1)(2t+1)\over 6(L^2+c^2)}.
  \end{align*}
  Now adding up over $t=1$ to $T$, we have
  \begin{equation*}
  \sum_{t=1}^T \E[\norm{x(t+1)-x(t-\tau_t)}^2] 
  \le C_1 + C_2\log T \qedhere
  \end{equation*}

\end{proof}
% Lemma gamma2
\begin{lemma}
  \label{lem:gamma2}
  Assuming scaled delays, we have the bound
  \begin{equation*}
    \begin{split}
      \sum_{t=1}^T \E[\Gamma(t)] &= \sum_{t=1}^T\E\left[\ip{\nabla f(x(t))-\nabla f(x(t-\tau_t))}{x(t+1) -x^*}\right]\\
      &\le GR\left(1+ \sum_{t=1}^{T-1}\frac{B_t^2}{(T-t)^2}\right) + LG^2\sum_{t=1}^T\frac{B_t^2+1+\btau_t}{L^2+c^2(1-\theta_t)t}.
%      &\le GR\left(1+ \frac{\pi^2B^2}{6}\right) + \frac{G^2(B^2+1+\tilde \tau)}{c^2(1-\theta)} \log T
    \end{split}
  \end{equation*}
\end{lemma}
\begin{proof}
  We build on Corollary 1 of~\citep{agDuc11}, and proceed as in Lemma~\ref{lem:gamma1} to bound the  terms in~(\ref{eq:12}) separately.  For the first term, we bound the expected cardinality using Chebyshev's inequality and Assumption~\ref{ass:delay}~(B):
  \begin{equation*}
    \E[|\set{t: t + \tau_t > T}|] = \sum_{t=1}^T\text{Pr}(\tau_t>T-t) \le 1 + \sum_{t=1}^{T-1}\frac{\E[\tau_t^2]}{(T-t)^2}
    = 
    1+ \sum_{t=1}^{T-1}\frac{B_t^2}{(T-t)^2} 
%    \le 1+ \frac{\pi^2B^2}{6}.
  \end{equation*}
  To bound the second term, we again follow Lemma~\ref{lem:gamma1} to obtain 
  \begin{equation*}
    \begin{split}
      \E[\norm{x(t+1)-x(t-\tau_t)}^2 | \tau_t] &\le (\tau_t+1)G^2\nlsum_{s=0}^{\tau_t}\E[\alpha(t-s,\tau_{t-s})^2].
  \end{split}
  \end{equation*}

    \begin{align*}
    \E[\alpha(t-s,\tau(t-s))^2] &= \E[\frac{1}{L^2+c^2((t-s)+\tau(t-s))+ 2Lc\sqrt{t-s+\tau(t-s)}}]\\
    &\le \frac{1}{L^2+c^2(t-s)},
  \end{align*}
  which yields the bound (since $0 \le s \le \tau_t$)
  \begin{equation*}
    \E[\norm{x(t+1)-x(t-\tau_t)}^2 | \tau_t] \le \frac{G^2(\tau_t+1)^2}{L^2+c^2(t-\tau_t)}
  \end{equation*}
  Now adding up over $t=1$ to $T$ consider
  \begin{equation*}
    G^2\nlsum_{t=1}^T\frac{(\tau_t+1)^2}{L^2+c^2(t-\tau_t)},
  \end{equation*}
  so that taking expectation (over $\tau_t$) we then obtain
  \begin{equation*}
    \sum_{t=1}^T\E[\norm{x(t+1)-x(t-\tau_t)}^2]
    \le G^2\sum_{t=1}^T\E\left[\frac{(\tau_t+1)^2}{L^2+c^2(t-\tau_t)}\right].
  \end{equation*}
  Using our assumption that $\tau_t < \theta_t t$ for $\theta_t \in (0,1)$, we have in particular that
  \begin{align*}
    & G^2\sum_{t=1}^T\E\left[\frac{(\tau_t+1)^2}{L^2+c^2(t-\tau_t)}\right] \\
    \le &
    G^2\sum_{t=1}^T\frac{1}{L^2+c^2(1-\theta_t)t}\E[(\tau_t+1)^2]\\
    \le &
    G^2\sum_{t=1}^T\frac{B_t^2+1+\btau_t}{L^2+c^2(1-\theta_t)t}\qedhere
%    \le & \frac{G^2(B^2+1+\tilde \tau)}{c^2(1-\theta)} \log T \qedhere
  \end{align*}
\end{proof}

% Lemma sigma1
\begin{lemma}
  \label{lem:sigma}
  Let the step-offsets be $\eta(t,\tau_t) = c\sqrt{t+\tau_t}$. For any delay distribution we have
  \begin{equation*}
    \sum_{t=1}^T\E[\Sigma(t)] \le \frac{\sigma^2}{c}\sqrt{T}.
  \end{equation*}
\end{lemma}
\begin{proof}
  From Assumption~\ref{ass:var} on the variance of stochastic gradients, it follows that
  \begin{equation*}
    \E[\Sigma(t)] = \E\left[\tfrac{1}{2\eta(t,\tau_t)}\norm{\nabla f(x(t-\tau_t))-g(t-\tau_t)}^2\right]
    \le \frac{\sigma^2}{2}\E\left[\eta(t,\tau_t)^{-1}\right].
  \end{equation*}

  Plugging in $\eta(t,\tau_t) = c\sqrt{t+\tau_t}$, clearly the bound
  \begin{equation}
    \label{eq:27}
    \frac{1}{c}\E[(t+\tau_t)^{-1/2}] = \frac{1}{c}\sum_{s=0}^{t-1}P(s)\frac{1}{\sqrt{t+s}} \le \frac{1}{c\sqrt{t}},
  \end{equation}
  holds for any delay distribution. Summing up over $t$, we then obtain
  \begin{equation*}
    \sum_{t=1}^T\E[\Sigma(t)] \le \frac{\sigma^2}{2c}\sum_{t=1}^T\frac{1}{\sqrt{t}}
    \le
    \frac{\sigma^2}{c}\sqrt{T}.\qedhere
  \end{equation*}
  \iffalse
  % stuff below not needed
  However, let us make a more precise statement, though the overall bound on $\sum_{t=1}^T\E[\Sigma(t)]$ is ultimately still $\Oc(\sqrt{T})$, the constants are better as a result.

  \emph{Uniform delays:}\\
  \begin{align*}
    \frac{1}{c}\E[(t+\tau_t)^{-1/2}] &= \frac{1}{c}\sum_{s=0}^{2\btau}\frac{1}{2\btau+1}\frac{1}{\sqrt{t+s}}\\
    &= \frac{1}{c\sqrt{t}} + \frac{1}{c(2\btau+1)}\sum_{s=1}^{2\btau}\frac{1}{\sqrt{t+s}}\\
    &\le \frac{1}{c\sqrt{t}} + \frac{1}{c(2\btau+1)}\frac{4\btau}{\sqrt{2\btau+2t}}.
  \end{align*}
  Consequently, summing over $t=1$ to $T$, we obtain the bound
  \begin{align*}
    \sum_{t=1}^T\E[\Sigma(t)] &\le \frac{\sigma^2}{2c}\sum_{t=1}^T\left(\frac{1}{\sqrt{t}} + \frac{4\btau}{(2\btau+1)}\frac{1}{\sqrt{2\btau+2t}}\right)\\
    &\le \frac{\sigma^2}{2c}\left(1 + 2\sqrt{T} + \frac{4\btau}{(2\btau+1)\sqrt{2\btau+2}} + \sqrt{2T+2\btau}\right)\\
    &= \frac{\sigma^2}{2c}\left(3 + 2\sqrt{T}+\sqrt{2T+\btau}\right).
  \end{align*}
  \emph{Scaled delays:}\\
  \fi

\end{proof}

\section{More general step-sizes}
If we use the offsets $\eta_t = c(t+\tau_t)^\beta$, where $0 < \beta < 1$, we obtain slightly more general step sizes that fit within our framework. The \emph{only} benefit of considering stepsizes other than $\beta=1/2$ is because they allow us to tradeoff the contributions of the various terms in the bounds, and for a larger value of $\beta$ for instance, we will obtain smaller step sizes, which can be beneficial in high noise regimes, at least in the initial iterations. The theoretical sweet-spot (in terms of dependence on $T$), is, however $\beta=1/2$, the choice analyzed above. We summarize below the impact of these steps sizes for non-uniform scaled delays; the uniform case is even simpler. For simplicity, we do not bound the terms as tightly as for the special case $\beta=1/2$.

\begin{lemma}
  \label{lem:beta}
  Assume that $\tau_t$ satisfies Assumption~\ref{ass:delay}~(B) and $\eta_t = c(t+\tau_t)^\beta$ and $0 < \beta < 1$. Then,
  \begin{align}
    \label{eq:4}
    \E[z_t^+] &\le \frac{cR^2\beta(\btau_t+1)}{(t-1)^{1-\beta}}\\
    \label{eq:5}
    \E[\norm{x(t+1)-x(t-\tau_t)}^2] &\le \frac{G^2(\tau_t+1)^2}{L^2+c^2(t-\tau_t)^{2\beta}}\\
    \label{eq:11}
    \E[\eta(t,\tau_t)^{-1}] &\le \frac{1}{ct^\beta}.
  \end{align}
\end{lemma}
\begin{proof}
  Proceeding as in Lemma~\ref{lem:delta2} we bound
  \begin{align*}
    \E[z_t^+] &= c\sum_{l=0}^{t-1}\sum_{s=0}^l P_t(l)P_{t-1}(s)\left((t+l)^\beta-(t-1+s)^\beta\right)\\
    &\le c\sum_{l=0}^{t-1}\sum_{s=0}^l P_t(l)P_{t-1}(s)\beta\frac{l+1-s}{(t-1+s)^{1-\beta}}\\
    &\le c\beta \sum_{l=0}^{t-1}\sum_{s=0}^l P_t(l)P_{t-1}(s)\frac{l+1}{(t-1)^{1-\beta}}\\
    &\le c\beta \sum_{l=0}^{t-1}P_t(l)\frac{l+1}{(t-1)^{1-\beta}} = \frac{c\beta(\btau_t+1)}{(t-1)^{1-\beta}}.
  \end{align*}
  where the first inequality follows from concavity if $t^\beta$, the second one since $\frac{l+1-s}{(t-1+s)^{1-\beta}}$ is decreasing in $s$, while the third is clear as $P_{t-1}$ is a probability.

  Next, we bound~\eqref{eq:5}. Proceeding as in Lemma~\ref{lem:gamma2}, we obtain the bounds
  \begin{align*}
    &\E[\alpha(t-s,\tau_{t-s})^2] \le \frac{1}{L^2+c^2(t-s)^{2\beta}}\\
    \implies &\E[\norm{x(t+1)-x(t-\tau_t)}^2|\tau_t] \le \frac{G^2(\tau_t+1)^2}{L^2+c^2(t-\tau_t)^{2\beta}}
  \end{align*}

  Finally, the bound on~\eqref{eq:11} is trivial; since $\eta_t^{-1} = c^{-1}(t+\tau_t)^{-\beta}$, we have
  \begin{equation*}
    \frac{1}{c}\E[(t+\tau_t)^{-\beta}] = \frac{1}{c}\sum_{s=0}^{t-1}P_t(s)\frac{1}{(t+s)^\beta} \le \frac{1}{ct^\beta}.\qedhere
  \end{equation*}
\end{proof}
Using these key bounds, we can defined full versions of Lemmas~\ref{lem:delta2}, \ref{lem:gamma2}, and \ref{lem:sigma}, where we finally we will need a bound of the form
\begin{equation*}
  \sum_{t=1}^T\frac{1}{t^\beta} \le 1 + \int_0^T t^{-\beta}dt = 1 + \frac{\left(T^{1-\beta}-1\right)}{1-\beta} \le \frac{1}{1-\beta}T^{1-\beta}.
\end{equation*}

\end{appendices}

\end{document}